\documentclass[11pt]{article}
\usepackage[letterpaper,margin=1in]{geometry}
\usepackage[round,authoryear]{natbib}

\usepackage[utf8]{inputenc} 
\usepackage[T1]{fontenc}    
\usepackage{graphicx}       
\graphicspath{{./}}
\usepackage{subcaption}     
\usepackage{hyperref}       
\hypersetup{hidelinks}
\usepackage{url}            
\usepackage{booktabs}       
\usepackage{amsfonts}       
\usepackage{microtype}      
\usepackage{xcolor}         
\usepackage{soul}
\usepackage{bm}
\usepackage{amsmath, amssymb, amsthm}

\newtheorem{proposition}{Proposition}

\title{Efficient Neural Field Learning via Adaptive Coverage and Focused Sampling}
\date{}

\author{%
  \textbf{Guang Zhao} \quad \textbf{Xihaier Luo} \quad \textbf{Huan-Hsin Tseng} \quad \textbf{Seungjun Lee} \\
  \textbf{Shinjae Yoo} \quad \textbf{Yihui Ren} \quad \textbf{Wei Xu} \\
  Computing and Data Sciences, Brookhaven National Laboratory \\
  Upton, NY 11973, USA \quad \url{https://www.bnl.gov/compsci/} \\
  \texttt{gzhao@bnl.gov, xluo@bnl.gov, htseng@bnl.gov, slee19@bnl.gov} \\
  \texttt{sjyoo@bnl.gov, yren@bnl.gov, xuw@bnl.gov}
}

\begin{document}

\maketitle

\begin{abstract}
    Implicit neural representations (INRs) provide a flexible framework for modeling high-dimensional continuous fields, but their training is often inefficient due to uniform subsampling that ignores spatial heterogeneity. Existing adaptive sampling methods partially address this issue by prioritizing high-error samples, but typically operate at the point level, often leading to redundant sampling in localized regions and insufficient coverage of the domain. We propose \textbf{ACES (Adaptive Coverage-aware Efficient Sampling)}, a structured sampling framework that improves training efficiency by decoupling coverage and importance. ACES constructs adaptive spatial partitions to ensure domain coverage and reduce redundancy, and applies region-level importance weighting to prioritize informative regions during training. We provide a theoretical analysis showing that adaptive partitioning reduces gradient variance by increasing within-region homogeneity, and that controlled bias in region-level weighting may improve optimization efficiency relative to standard unbiased estimators. Experiments on scientific field learning tasks demonstrate that ACES achieves faster convergence and lower error than uniform and pointwise adaptive sampling baselines, with the largest gains in fields with highly localized complexity.
\end{abstract}

\section{Introduction}
Modern scientific discovery increasingly relies on accurate, efficient, and flexible representations of high-dimensional continuous fields, such as spatiotemporal physical states, material microstructures, and environmental transport processes~\citep{park2025scent, qu2026representation}. Many of these systems are naturally defined on continuous, irregular, or evolving domains, where measurements and simulations provide large collections of pointwise evaluations rather than a single fixed grid. This has motivated implicit neural representations (INRs), a growing class of machine learning methods that optimize pointwise objectives over continuous spaces~\citep{sitzmann2020implicit, mildenhall2020nerf, park2019deepsdf}. These methods are flexible and resolution-independent, naturally handle complex geometry, remain compatible with gradient-based optimization, and are well suited to surrogate modeling, inverse problems, and accelerating simulation-driven science.

Despite their conceptual elegance, INRs often suffer from significant training inefficiency when learning continuous fields. In many practical settings, millions to billions of candidate points are available from simulations, sensors, or trajectory evaluations of the field, yet computational and memory constraints limit how many can be used per iteration. Uniform sampling from this large pool treats all regions of the field as equally informative, while in many scientific systems, complexity is not uniformly distributed: sharp interfaces, localized dynamics, and multi-scale structures lead to highly uneven error landscapes during training. Uniform sampling fails to account for this heterogeneity, allocating excessive computation to simple regions while under-resolving the complex regions that dominate approximation error.

This observation naturally motivates adaptive sampling, where training points are selected or reweighted to prioritize informative regions and improve efficiency. A range of methods have been proposed along this direction, typically using local criteria such as prediction error, residual magnitude, or gradient norm to guide sampling~\citep{zhang2024nonparametric, kheradmand2024accelerating, zhang2025evos, tack2023learning}. While these approaches can accelerate convergence in simple settings, they remain largely heuristic and exhibit two fundamental limitations. First, most treat samples independently and ignore the spatiotemporal structure of the loss field, often causing over-concentration in localized regions where nearby points provide redundant training signals, at the expense of global coverage. Second, while many methods implicitly introduce biased gradient estimators through non-uniform sampling~\citep{zhang2024nonparametric, kheradmand2024accelerating}, this bias is neither explicitly analyzed nor systematically exploited. In practice, such bias often accelerates convergence, yet existing work lacks a principled understanding of why this occurs or how to design it effectively.

To address these issues, we propose \textbf{ACES (Adaptive Coverage-aware Efficient Sampling)}, a structured adaptive sampling framework that explicitly accounts for the heterogeneity of the loss field. ACES combines structure-aware domain partitioning with bias-controlled sampling. First, ACES shifts from point-level to region-level sampling: it adaptively partitions the domain according to the evolving loss structure, enabling sampling decisions that respect spatial organization rather than purely local signals. This partitioning enforces coverage while reducing redundancy, preventing excessive clustering. Second, within each adaptive region, ACES further incorporates a bias-controlled importance weighting mechanism to prioritize informative samples. We provide theoretical insights into the role of bias in gradient estimation, suggesting how controlled bias may improve optimization efficiency, and use this analysis to guide the design of our estimator. Together, these components yield a sampling strategy that balances coverage, variance reduction, and fast convergence for learning high-dimensional continuous fields.

Figure~\ref{fig:intro_sampling_comparison} illustrates the key intuition behind ACES. Uniform sampling provides broad coverage but allocates many samples to low-error regions, while pointwise importance sampling over-concentrates in localized high-error areas. In contrast, ACES performs adaptive region-level sampling that maintains domain coverage while allocating larger sampling weights to informative regions.

\begin{figure}[t]
  \centering
  \includegraphics[width=\linewidth]{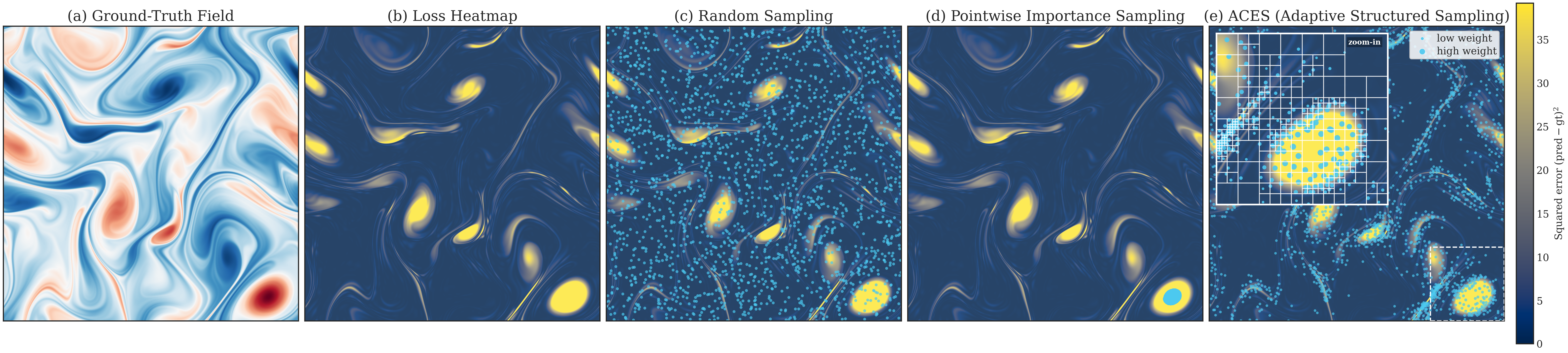}
  \caption{Qualitative comparison of sampling strategies on a Navier–Stokes vorticity field at a training iteration. Uniform sampling provides broad coverage but allocates many samples to low-error regions, while pointwise importance sampling over-concentrates in localized high-error areas. ACES performs adaptive region-level sampling, balancing domain coverage and focused refinement. Point size indicates the region-level sampling weight.}
  \label{fig:intro_sampling_comparison}
  \vspace{-5mm}
\end{figure}

\vspace{-2mm}
\section{Related Work}

\paragraph{Adaptive sampling for implicit neural representations.} 
Given the continuous coordinate-based nature of INRs, prior work has explored adaptive sampling for INR training by prioritizing samples with high residuals, uncertainty, or other local importance signals~\citep{zhang2024nonparametric, zhang2025evos, kheradmand2024accelerating, tack2023learning}. These approaches primarily operate at the individual point level and largely ignore higher-level spatial structure. ACES differs by using adaptive partitioning and region-level weighting to better balance coverage and focus.

\paragraph{Variance reduction through structured mini-batches.}
Variance reduction in stochastic optimization has been widely studied through importance sampling and structured mini-batch construction. Importance-sampling methods prioritize informative samples using smoothness, gradient norm, or loss-based criteria \citep{needell2014sgd,katharopoulos2018not}, while stratified sampling reduces variance by partitioning data into more homogeneous groups \citep{zhao2014accelerating}. However, these methods are primarily designed for fixed datasets with independently sampled examples and do not exploit the spatial structure of continuous neural fields. In contrast, ACES adaptively partitions continuous domains to jointly improve coverage and variance reduction.

\paragraph{Weighted sampling and controlled bias.}
Beyond unbiased variance reduction, recent studies have shown that optimization performance depends not only on raw gradient variance but also on the relationship between gradient variance and update quality~\citep{faghri2020study}. ACES performs adaptive region-level sampling over continuous domains and explicitly studies the role of controlled bias through region-level weighting and gradient-alignment analysis.

\section{Methodology}

\subsection{Problem Setting}
We consider the problem of learning a continuous field from a large set of pointwise evaluations. Let $\Omega \subset \mathbb{R}^d$ denote a spatial or spatiotemporal domain. We assume access to a dataset
$\mathcal D=\{(\mathbf x_i,\mathbf y_i)\}_{i=1}^N$
with $\mathbf x_i\in\Omega$ and $\mathbf y_i\in\mathbb R^m$, 
where $\mathbf{y}_i$ may represent physical quantities such as velocity, pressure, or vorticity. We aim to learn a coordinate-based model (implicit neural representation)
$
f_\theta: \Omega \to \mathbb{R}^m,
$
parameterized by $\theta$, that approximates the underlying continuous field. Consider the empirical risk:
\begin{equation}
\widehat{\mathcal{L}}(\theta) = \frac{1}{N} \sum_{(\mathbf{x}_i, \mathbf{y}_i) \in \mathcal{D}} \ell\big(f_\theta(\mathbf{x}_i), \mathbf{y}_i\big),
\end{equation}
where $\ell(\cdot)$ is a pointwise loss function. The full gradient is
$\mathbf g=\nabla_\theta\widehat{\mathcal L}(\theta)
=\frac1N\sum_{i=1}^N\mathbf g_i$,
where
$\mathbf g_i=\nabla_\theta\ell(f_\theta(\mathbf x_i),\mathbf y_i)$.
In many scientific applications, the dataset size $N$ can be extremely large, making full-gradient computation expensive in both computation and memory. In practice, we instead use a minibatch $\mathcal{B}$ of size $n \ll N$ at each iteration, yielding the stochastic gradient estimator:
\begin{equation}
\hat{\mathbf{g}}_{\mathrm{RS}} = \frac{1}{n} \sum_{i=1}^{n} \mathbf{g}_i.
\end{equation}

This estimator is effective when gradient contributions are relatively homogeneous across samples. However, in many physical systems, the learning signal is highly heterogeneous, leading to inefficient use of computation under uniform sampling.

To improve training efficiency, we consider sampling minibatches according to a non-uniform distribution. Let $q(i)$ denote the probability of selecting sample $i$ from the dataset $\mathcal{D}$. A general stochastic gradient estimator under distribution $q$ takes the form
\begin{equation}
\hat{\mathbf{g}}_{q}
=
\frac{1}{n} \sum_{i=1}^{n} \frac{w_i}{q(i)} \mathbf{g}_i,
\qquad i \sim q,
\end{equation}
where $w_i$ is a weight associated with sample $i$. When $w_i = \frac{1}{N}$, the estimator reduces to standard importance sampling and is an unbiased estimator of the full gradient. More generally, we allow $w_i$ to deviate from $\frac{1}{N}$, resulting in a biased estimator:
$\bar{\mathbf g}=\mathbb E[\hat{\mathbf g}_q]\neq\mathbf g$.
This perspective treats adaptive sampling not as a mechanism for strictly approximating the full gradient, but as a way to potentially accelerate optimization by prioritizing informative samples.


Existing adaptive sampling methods operate at the level of individual points by assigning importance scores to each $\mathbf{x}_i \in \mathcal{D}$, based on residual or gradient statistics. While effective in some settings, pointwise sampling may over-concentrate in localized regions, leading to redundant updates and poor domain coverage. It can also incur substantial overhead due to frequent importance estimation over large candidate sets. These limitations motivate region-level adaptive sampling.

\paragraph{ACES Overview} 
As illustrated in Fig.~\ref{fig:diagram}, ACES performs adaptive region-level sampling over the continuous domain. It begins with a coarse partition and recursively refines regions that remain highly variable, producing more homogeneous regions in which sampled points better represent their local neighborhoods. Sampling and weighting are then applied at the region level rather than independently per point, improving coverage while reducing redundant sampling in spatially correlated areas. Regions with greater remaining variability receive larger sampling budgets, while samples from regions with poorer predictions are assigned greater weights.

ACES uses a two-timescale implementation. 
The partition is updated periodically on a slow timescale, while region weights are updated every iteration using lightweight residual statistics. This is because the spatial structure of the loss field evolves gradually during training, while the sampling weights directly influence which regions contribute most strongly to the parameter update. This separation amortizes partitioning cost while allowing the sampling distribution to adapt online.

The following sections answer two questions: First, how should the domain be partitioned so that stochastic gradient estimates have low variance? Second, how should samples or weights be assigned across regions when the goal is not only unbiased gradient estimation, but also faster optimization? We address these questions through a variance-reduction analysis of adaptive partitioning, followed by a bias--variance analysis of region-level weighted sampling.

\begin{figure}[ht]
    \centering
    \includegraphics[width=0.8\textwidth]{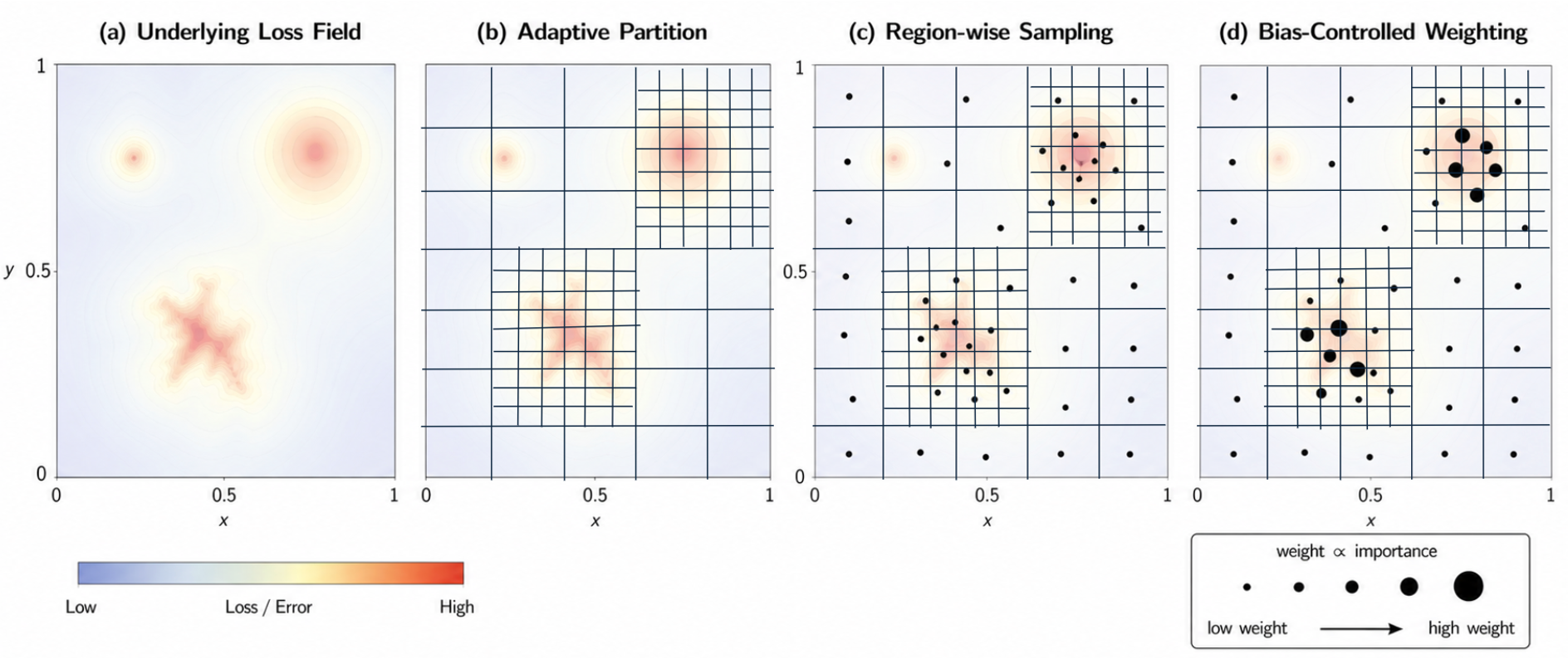}
    \caption{Overview of ACES. (a) The loss field is highly heterogeneous, with localized high-error regions.
(b) ACES builds an adaptive partition that allocates finer resolution to complex regions and coarser resolution elsewhere.
(c) Region-based sampling ensures structured coverage of the domain while avoiding redundant samples.
(d) Bias-controlled weighting emphasizes informative regions by assigning larger weights to high-error areas, leading to more effective optimization.}
    \label{fig:diagram}
\end{figure}

\subsection{Variance Reduction via Partitioning}

To connect the sampling strategy with optimization efficiency, we analyze the stochastic gradient estimators induced by different sampling schemes through a one-step optimization analysis.



Assume $\mathcal{L}$ is $\beta$-smooth, i.e.,
\(
\|\nabla \mathcal{L}(\theta_1) - \nabla \mathcal{L}(\theta_2)\| \le \beta \|\theta_1 - \theta_2\|, \, \forall \theta_1, \theta_2
\). For the update
$\theta^+ = \theta - \eta \hat{\mathbf g}$ with
$\mathbf g = \nabla \mathcal{L}(\theta)$, the standard smoothness inequality gives~\citep{zhao2014accelerating}:
\begin{equation}
\mathbb{E}[\mathcal{L}(\theta^+)]
\le
\mathcal{L}(\theta)
-
\eta \langle \mathbf g, \bar{\mathbf g} \rangle
+
\frac{\beta \eta^2}{2}
\mathbb{E}\|\hat{\mathbf g}\|^2,
\label{eq:one_step_improvement}
\end{equation}
where $\mathbb{E}(\cdot)$ denotes expectation over stochastic gradient and $\bar{\mathbf g}=\mathbb{E}[\hat{\mathbf g}]$.
Using
\begin{equation}
\mathbb{E}\|\hat{\mathbf g}\|^2
=
\|\bar{\mathbf g}\|^2
+
\mathbb{E}\|\hat{\mathbf g}-\bar{\mathbf g}\|^2,
\end{equation}
we see that in the unbiased case
($\bar{\mathbf g}=\mathbf g$), local optimization behavior is influenced primarily by the variance term
$\mathbb{E}\|\hat{\mathbf g}-\mathbf g\|^2$.
This motivates constructing partitions that reduce the variance of the stochastic gradient estimator.

Under uniform random sampling, the minibatch estimator has variance
\begin{equation}
\mathbb{E}\|\hat{\mathbf g}_{\mathrm{RS}}-\mathbf g\|^2
=
\frac{\sigma^2}{n},
\end{equation}
where $\sigma^2$ is the variance of per-sample gradients.

To reduce this variance, we partition the dataset into $K$ disjoint regions, where region $k$ contains $N_k$ samples and $\sum_{k=1}^K N_k = N$. 
From each region $k$, we draw $n_k$ samples, with $\sum_{k=1}^K n_k = n$. The resulting stratified gradient estimator is
\begin{equation}
\hat{\mathbf{g}}_{\mathrm{st}} 
= \sum_{k=1}^{K} \frac{N_k}{N} \cdot \frac{1}{n_k} 
\sum_{i \in \mathcal{B}_k} \mathbf{g}_i,
\end{equation}
where $\mathcal{B}_k \subset \mathcal{D}_k$ denotes the sampled subset from region $k$. 
This estimator is unbiased, and its variance is given by \citep{botev2017variance}:
\begin{equation}
\mathbb{E}\left[\|\hat{\mathbf{g}}_{\mathrm{st}} - \mathbf{g}\|^2\right] 
= \sum_{k=1}^{K} \left(\frac{N_k}{N}\right)^2 \frac{1}{n_k} \sigma_k^2,
\label{eq:stratified}
\end{equation}
where $\sigma_k^2$ denotes the within-region gradient variance. For a fixed partition, the optimal allocation that minimizes the estimator variance is given by the Neyman allocation:
\begin{equation}
n_k  \propto N_k \sigma_k, 
\end{equation}

Based on Eq.~\ref{eq:stratified}, variance reduction depends critically on constructing regions with low internal gradient heterogeneity. We therefore use an \emph{adaptive partitioning} strategy that recursively refines high-variance regions, producing more homogeneous regions and further reducing estimator variance.

\paragraph{Partition Rule.}
To guide adaptive refinement, we estimate the variance contribution of each region under stratified sampling. For a region $k$ containing $N_k$ samples and assigned sampling budget $n_k$, the contribution to the estimator variance is
\begin{equation}
\Delta V_k
=
\left(\frac{N_k}{N}\right)^2
\frac{1}{n_k}
\sigma_k^2.
\end{equation}

To relate adaptive refinement to variance reduction, we assume that gradients vary smoothly within a spatially coherent region, so that nearby samples produce more similar gradients than distant samples. Specifically, within region $k$, following a local variogram model for the gradient field~\citep{cressie2015statistics},
\begin{equation}
\mathbb{E}\|\mathbf{g}_i-\mathbf{g}_j\|^2
\approx
a_k\|\mathbf{x}_i-\mathbf{x}_j\|^\nu,
\qquad
\mathbf{x}_i,\mathbf{x}_j \in \mathcal{D}_k,
\end{equation}
where $a_k>0$ captures the local roughness of the sample-gradient field and $\nu>0$ describes how gradient discrepancy grows with spatial distance.
Under this scaling, subdividing a region reduces its characteristic spatial diameter and is therefore expected to reduce its within-region gradient variance. 
Thus, the expected benefit of refining region $k$ is proportional to its current variance contribution,
\begin{equation}
\Delta V_k - \Delta V_k^{\mathrm{split}}
\propto
\Delta V_k.
\end{equation}
The proof of this relationship is included in Appendix~\ref{app:variance_reduction}.
Under Neyman allocation, the sampling budget satisfies $n_k \propto N_k\sigma_k$. 
Substituting this relation into $\Delta V_k$ gives
\begin{equation}
\Delta V_k
\propto
\frac{N_k^2}{N^2}
\frac{\sigma_k^2}{N_k\sigma_k}
=
\frac{1}{N^2}N_k\sigma_k.
\end{equation}
This motivates the refinement score
\begin{equation}
s_k = N_k\sigma_k,
\end{equation}
and ACES recursively subdivides regions with large $s_k$. 
The criterion favors regions of both large size (large number of points) and internally heterogeneous regions, making it a practical proxy for potential variance reduction.

\subsection{Bias--Variance Tradeoff and Acceleration}
\label{sec:bias_variance_tradeoff}

We now consider an additional mechanism for improving optimization efficiency: region-level weighting. Rather than minimizing the variance of an unbiased gradient estimator, we allow biased weights that emphasize informative regions and aim to yield larger expected local improvement.

Let
$
p_k = \frac{N_k}{N}$, 
$\bar{\mathbf{g}}_k
=
\frac{1}{N_k}
\sum_{i\in \mathcal{D}_k}
\mathbf{g}_i
$ and $
\hat{\mathbf{g}}_k
=
\frac{1}{n_k}
\sum_{i\in \mathcal{B}_k}
\mathbf{g}_i ,
$
where $\bar{\mathbf{g}}_k$ is the mean gradient in region $k$ and $\hat{\mathbf{g}}_k$ is its minibatch estimate.
We define the weighted stratified estimator as
\begin{equation}
\hat{\mathbf{g}}_{\mathrm{w}}
=
\sum_{k=1}^{K}
p_k w_k \hat{\mathbf{g}}_k
=
\sum_{k=1}^{K}
\frac{N_k}{N}
\frac{w_k}{n_k}
\sum_{i\in \mathcal{B}_k}
\mathbf{g}_i .
\label{eq:biased_estimator}
\end{equation}

For nonuniform $w_k$, the estimator is generally biased:
\(
\bar{\mathbf{g}}_{\mathrm{w}}
=
\mathbb{E}[\hat{\mathbf{g}}_{\mathrm{w}}]
=
\sum_{k=1}^{K}
p_k w_k \bar{\mathbf{g}}_k
\neq
\mathbf{g}.
\)

To understand the effect of this bias, we rewrite the one-step smoothness bound in Eq.~\ref{eq:one_step_improvement} as

\begin{equation*}
\mathbb{E}[\mathcal{L}(\theta^+)]
\le \mathcal{L}(\theta)-\eta_w S_w
+\tfrac{\beta\eta_w^2}{2}G_w,
\quad \text{where }\;
S_w:=\langle\mathbf{g},\bar{\mathbf{g}}_w\rangle,
\quad
G_w:=\mathbb{E}\|\hat{\mathbf{g}}_w\|^2.
\end{equation*}
Here, $S_w$ measures the alignment between the expected weighted estimator and the full gradient, while $G_w$ measures the estimator energy. For $S_w>0$ and $G_w>0$, minimizing the bound over $\eta_w$ yields:
\begin{equation}
\mathbb{E}[\mathcal{L}(\theta^+)]
\le
\mathcal{L}(\theta)
-
\frac{S_w^2}{2\beta G_w}.
\end{equation}
Thus, the ratio
\vspace{-2mm}
\begin{equation}
\frac{S_w^2}{G_w}
\label{eq:alignment_energy_ratio}
\end{equation}
acts as a local proxy for the expected one-step improvement, which can be used to compare biased gradient estimators.

To expose the tradeoff between directional alignment and relative
variance, we rewrite the improvement proxy as
\begin{equation}
\frac{S_w^2}{G_w}
=
\frac{
\left\langle
\mathbf{g},
\bar{\mathbf{g}}_w/\|\bar{\mathbf{g}}_w\|
\right\rangle^2
}{
1+
\mathbb{E}\|\hat{\mathbf{g}}_w-\bar{\mathbf{g}}_w\|^2
/\|\bar{\mathbf{g}}_w\|^2
}
=
\|\mathbf{g}\|^2
\frac{\cos^2\phi_w}{1+\rho_w}.
\end{equation}
Here, $\phi_w$ is the angle between the nonzero vectors
$\mathbf{g}$ and $\bar{\mathbf{g}}_w$, and
$\rho_w :=
\mathbb{E}\|\hat{\mathbf{g}}_w-\bar{\mathbf{g}}_w\|^2
/\|\bar{\mathbf{g}}_w\|^2$
is the relative variance.

At a fixed checkpoint, $\|\mathbf{g}\|^2$ is common to all estimators, so the proxy depends only on directional alignment and relative variance. For $S_w>0$, unbiased estimation maximizes the directional term, but need not minimize relative variance. A biased estimator can therefore achieve a higher proxy value if its reduction in relative variance outweighs the loss in squared directional alignment. We use inexpensive residual statistics to construct practical weights and analyze the resulting weighting family directly below.

\subsection{Practical Implementation}

The analysis above motivates two components of ACES: adaptive partitioning based on region-level gradient variability and region-level weighting based on informativeness. In practice, directly estimating gradient statistics for all regions is computationally expensive. We therefore use lightweight residual-based approximations that preserve the main design principles while remaining efficient.

\paragraph{Residual statistics as proxies for gradient variance.}

The partitioning criterion depends on the region score $N_k\sigma_k$, where $\sigma_k$ denotes the within-region gradient standard deviation. Since estimating $\sigma_k$ requires per-sample gradients, we instead approximate gradient variability using residual statistics.

For the scalar-output squared-loss setting used in our experiments, let $\ell_i=\tfrac12 r_i^2$, where
$
r_i = f_\theta(\mathbf x_i)-y_i,
$
and 
$
J_i = \nabla_\theta f_\theta(\mathbf x_i).
$
Then the per-sample gradient is
\[
\mathbf g_i = J_i^\top r_i .
\]
In this section, we denote the total variance of $\mathbf g_i$ within region $k$ as $\sigma_{g,k}^2$, which can be approximated as
\[
\sigma_{g,k}^2
\approx
\|\bar J_k\|_2^2\,\sigma_{r,k}^2
+
\bar r_k^2\,\sigma_{J,k}^2,
\]
where
$
\sigma_{r,k}^2
=
\mathbb E\!\left[
(r_i-\bar r_k)^2
\right]
$
is the variance of the residual $r_i$, and
$\sigma_{J,k}^2 = \operatorname{tr}(\operatorname{Cov}(J_i))$
is the variance of the Jacobian $J_i$ within the region. Here,
$\bar r_k = \mathbb{E}[r_i]$
and
$\bar J_k = \mathbb{E}[J_i]$
denote the mean residual and mean Jacobian over region $k$, respectively. This approximation is obtained by neglecting the cross term in the variance decomposition; the scalar-output derivation is provided in Appendix~\ref{app:grad_variance}. For vector outputs, an analogous decomposition uses residual and Jacobian covariance matrices. If we further assume within-region smoothness such that $J_i$ varies slowly in region $k$, we can ignore Jacobian variation inside the region and obtain
\(
\sigma_{g,k}^2
\approx
\|\bar J_k\|_2^2\,\sigma_{r,k}^2.
\)

In practice, we estimate the variance of the residual $\widehat{\sigma}_{r,k}^2$ within each region using Monte Carlo samples. 
We then define the practical partition score as
\begin{equation}
s_k^{\mathrm{part}}
=
N_k
\widehat{\sigma}_{r,k}
\label{eq:practical_partition_score}
\end{equation}
Regions with larger $s_k^{\mathrm{part}}$ are selected for refinement. 
This criterion favors regions that are both large and difficult to approximate, while avoiding the cost of computing explicit gradient variance.




\paragraph{Region-level weighting for biased sampling.}

The theoretical criterion in Eq.~\ref{eq:alignment_energy_ratio} suggests that regions should be emphasized when their gradients are well aligned with the full-gradient direction and do not introduce excessive variance. However, the full gradient is unavailable in stochastic training. We therefore use region-level statistics to design the weights. Intuitively, weights should be large for regions with a large mean residual $\bar{\mathbf r}_k$ and small variance $\sigma^2_{r,k}$. However, accurately estimating $\sigma^2_{r,k}$ at the sampling stage in every training iteration would require evaluating multiple samples per region, which introduces substantial overhead and becomes unreliable when some regions receive only a small number of samples. Therefore, we use only $\bar{\mathbf r}_k$ as a lightweight proxy for informativeness.
We define the unnormalized region weight as
\[
\tilde w_k
=
|\bar r_k|^\alpha,
\]
where $\alpha\ge 0$ controls the strength of the sampling bias. 
The normalized weight is
\begin{equation}
w_k
=
\frac{
|\bar r_k|^\alpha
}{
\sum_{j=1}^{K} p_j|\bar r_j|^\alpha
},
\label{eq:practical_region_weight}
\end{equation}
so that $\sum_{k=1}^{K} p_k w_k = 1$. When $\alpha=0$, Eq.~\eqref{eq:practical_region_weight} recovers the unbiased weighting $w=\mathbf{1}$. For this weighting family, we show in Appendix~\ref{app:alpha_monotone}
that
\begin{equation}
\left.
\frac{dF(w(\alpha))}{d\alpha}
\right|_{\alpha=0}
=
C\sum_{k=1}^{K}
t_k\log|\bar r_k|,
\qquad C>0,
\label{eq:residual_weight_derivative}
\end{equation}
where $t_k$ measures the relative alignment--variance advantage of
region $k$ (defined in Appendix~\ref{app:alpha_monotone}).
Hence, when the right-hand side is positive, residual-based weighting
follows a local ascent direction of $F$, and a sufficiently small
$\alpha>0$ strictly improves the proxy over unbiased weighting.
We verify this condition empirically and find
$\left.dF/d\alpha\right|_{\alpha=0}>0$ at all evaluated training
checkpoints (Appendix~\ref{app:alpha_monotone_empirical}).


Overall, ACES combines a slow-timescale adaptive partition with fast-timescale region-level weighting. The partition step improves sample representativeness by reducing within-region heterogeneity, while the weighting step accelerates optimization by biasing computation toward informative regions.

\section{Experiments}
We evaluate ACES on spatial and joint spatiotemporal field reconstruction tasks with heterogeneous structure. Our experiments examine sample efficiency, the trade-off between convergence and wall-clock cost, the complementary roles of adaptive partitioning and residual-aware weighting, and generalization from independent snapshots to joint space-time reconstruction.

\paragraph{Datasets and tasks.}
For spatial reconstruction, we consider Navier--Stokes vorticity fields at Reynolds number $\mathrm{Re}=10{,}000$ (NS2D; $1024\times1024$) and boiling-process fields from BubbleML~\citep{hassan2023bubbleml} (PoolBoiling2D; $384\times384$). The former contain multiscale vortices and thin shear layers, while the latter contain sparse localized interfaces in relatively smooth backgrounds. We uniformly select five temporal snapshots from each simulation trajectory and fit each snapshot independently. For joint spatiotemporal reconstruction, we train a single INR over continuous $(x,y,t)$ coordinates spanning 64 consecutive frames: PoolBoiling3D contains $384\times384\times64$ points, and NS3D contains $512\times512\times64$ points. Here, the third coordinate is time.

\paragraph{Sampling baselines.}
We compare ACES with uniform random sampling, Implicit Neural Teaching (INT)~\citep{zhang2024nonparametric}, EVOlutionary Selector (EVOS)~\citep{zhang2025evos}, and \emph{Fixed Stratified}, which uses a static uniform grid with a matched stratified allocation to isolate the benefit of adaptive partitioning. In the joint spatiotemporal experiments, EVOS is omitted because the evaluated implementation is not directly applicable to this setting.

\paragraph{Evaluation protocol.}
Our main experiments use a SIREN~\citep{sitzmann2020implicit} backbone and SGD. The spatial experiments sample a fraction $2\times10^{-3}$ of the candidate points per iteration, while the larger joint spatiotemporal experiments use $7\times10^{-4}$. Learning rates are selected by a coarse search for each dataset and method. Because each task reconstructs a known field, training draws coordinates from that field and evaluation computes relative mean squared error (relative MSE) over the full coordinate grid; no conventional train/test split is used. Relative MSE is normalized by the squared norm of the ground-truth field. Results are averaged over three random seeds, and tables report mean $\pm$ standard deviation. Appendix~\ref{sec:appendix_ffn_backbone} evaluates an additional Fourier Feature Network backbone; further implementation and training details are provided in Appendix~\ref{sec:appendix_experiment_setting_hardware}.

\subsection{Reconstruction Efficiency}
\paragraph{Spatial reconstruction.}
Figure~\ref{fig:exp_sampler_efficiency_combined} compares reconstruction error against training iterations and wall-clock time. ACES achieves the lowest final relative MSE on both datasets under both evaluation budgets. At equal iteration budgets, it outperforms uniform random sampling, fixed stratification, and the pointwise adaptive baselines, with particularly pronounced gains on PoolBoiling2D.

These improvements also translate into wall-clock benefits. On NS2D, uniform sampling initially reduces error faster, but ACES overtakes it later in training despite its sampling overhead. On PoolBoiling2D, ACES establishes an early advantage and maintains substantially lower error throughout most of training. Fixed stratification performs similarly to uniform sampling per iteration but offers no wall-clock advantage. The larger gains on PoolBoiling2D are consistent with the benefits of adaptive region-level sampling in fields with localized interfaces.

A detailed runtime breakdown on NS2D is reported in Appendix~\ref{sec:appendix_runtime_sensitivity}. It separates model training, sampling, and amortized partition-update costs under the current ACES configuration.


\begin{figure*}[t]
  \centering
\vspace{-1.5em}
  \begin{subfigure}[t]{0.89\textwidth}
    \centering
    \includegraphics[width=0.25\linewidth]{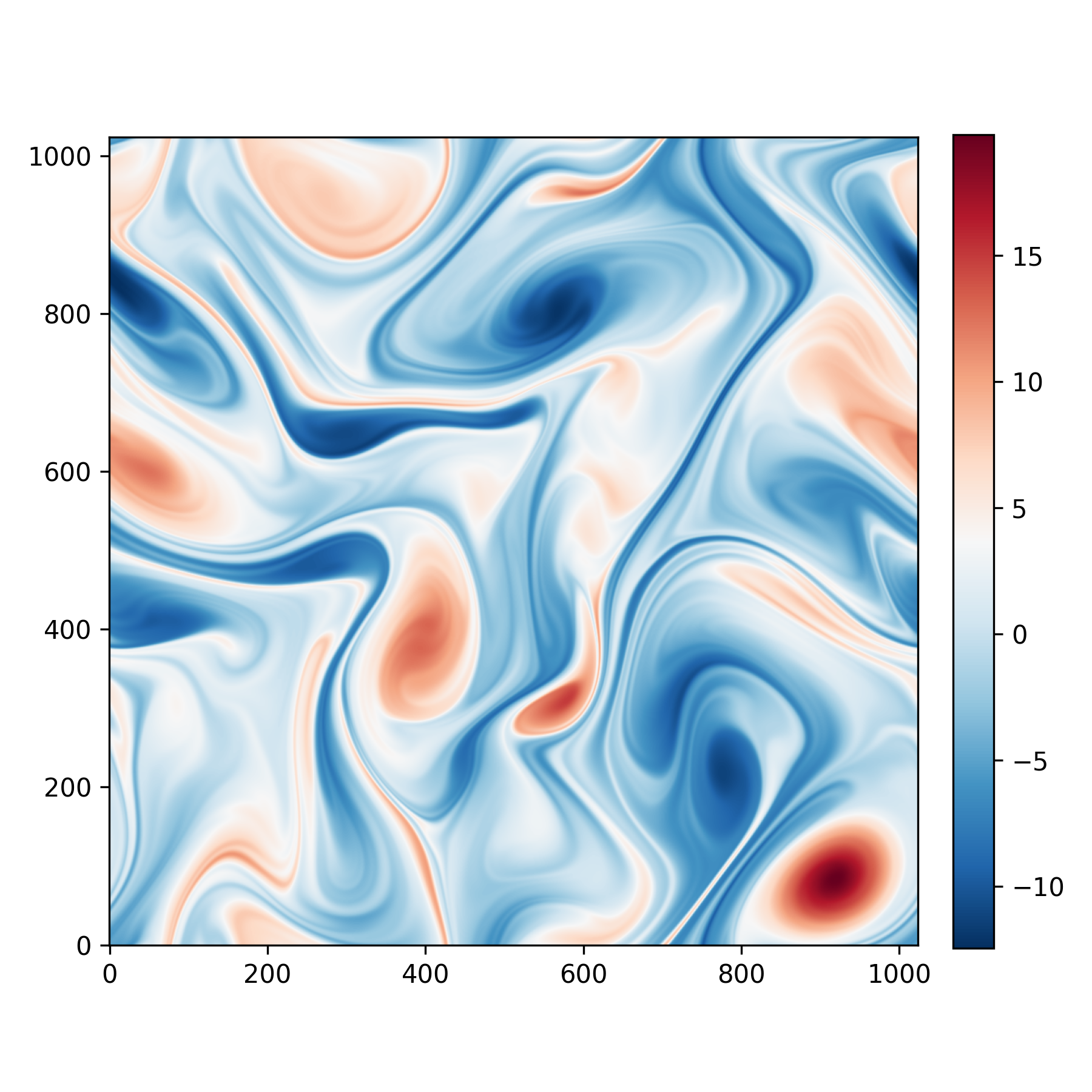}
        \hfill
    \includegraphics[width=0.35\linewidth]{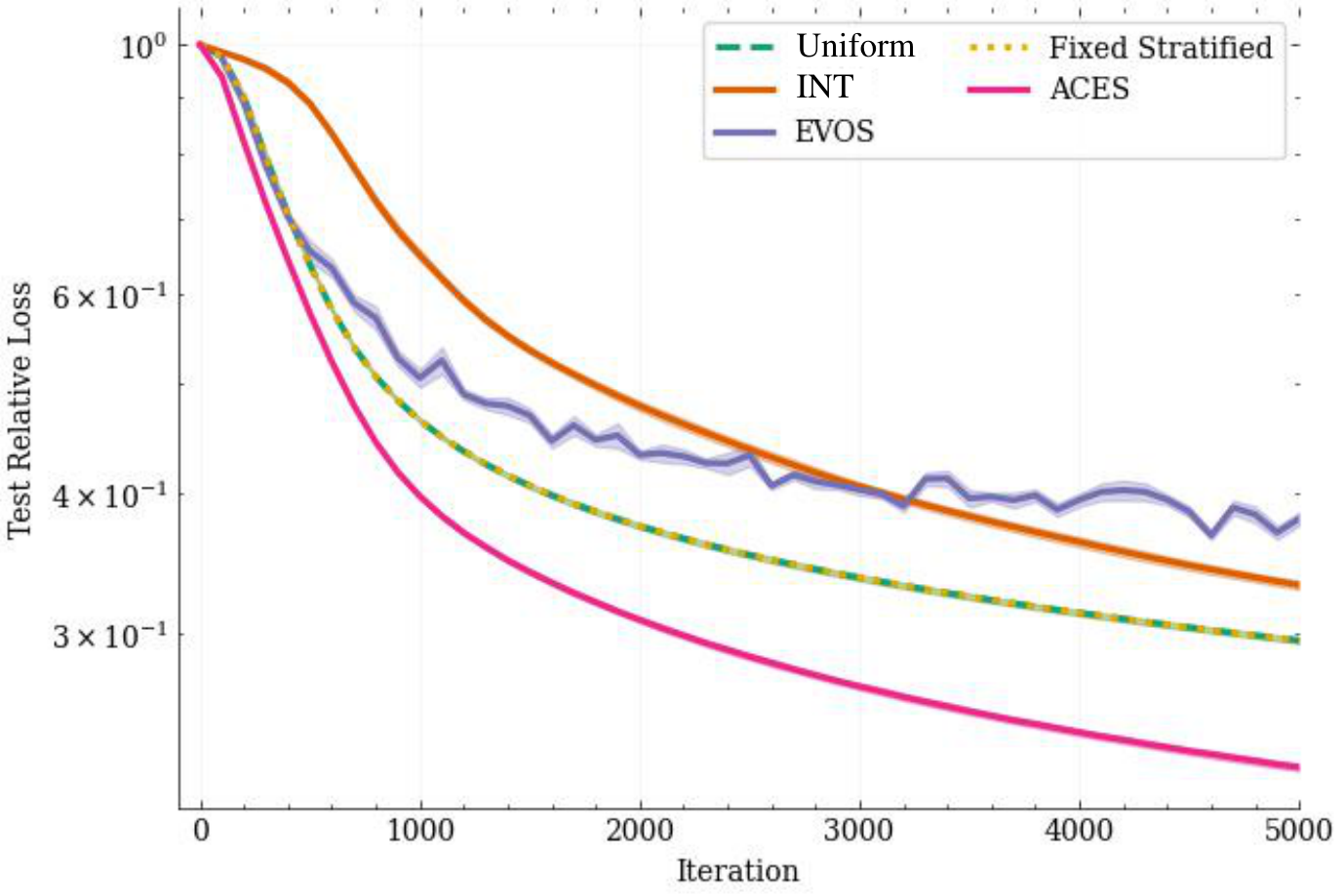}
    \hfill
    \includegraphics[width=0.35\linewidth]{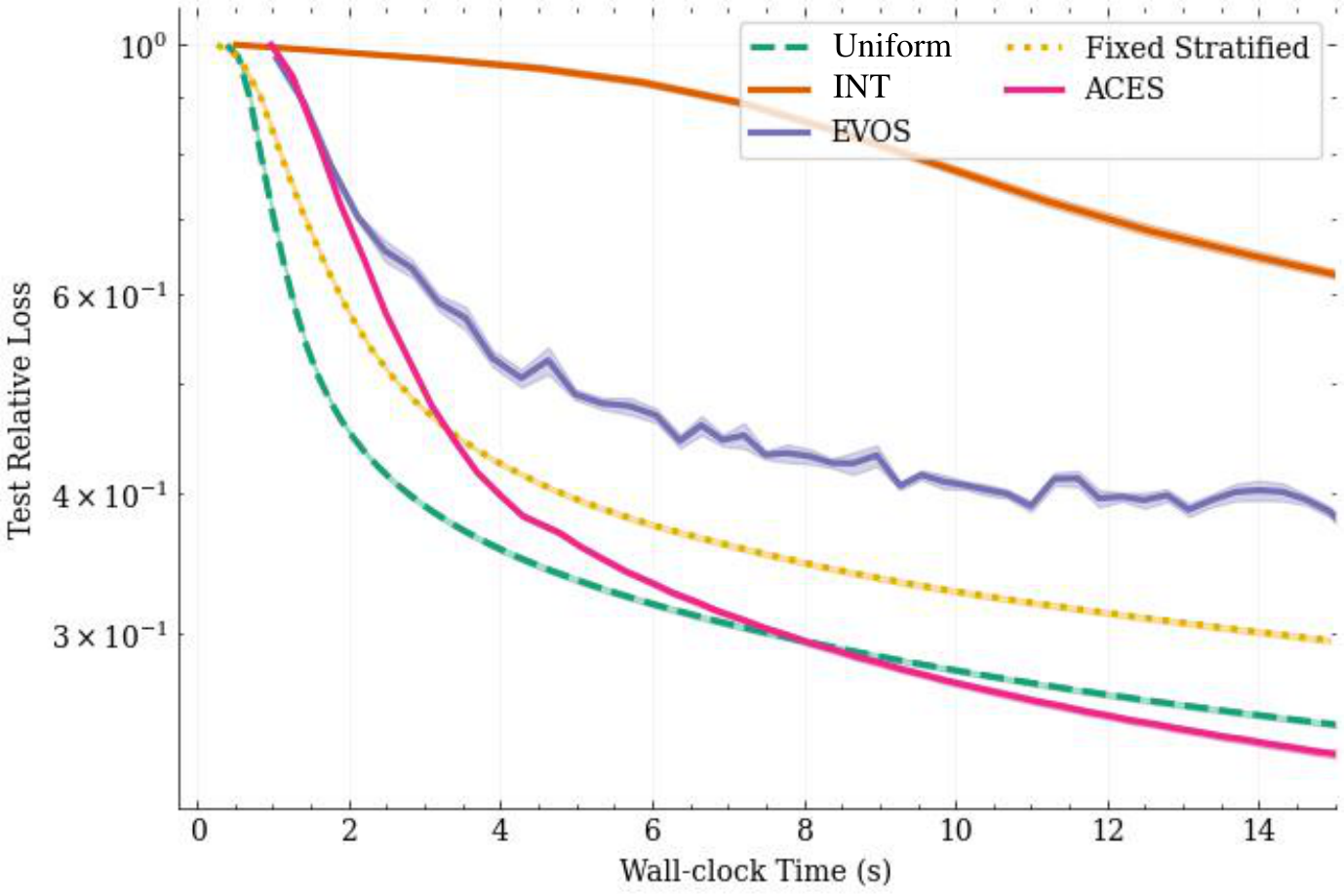}
    \caption{NS2D:distributed structures}
    \label{fig:ns1024_efficiency}
  \end{subfigure}

  \vspace{-0.5em}

  \begin{subfigure}[t]{0.89\textwidth}
    \centering
    \includegraphics[width=0.25\linewidth]{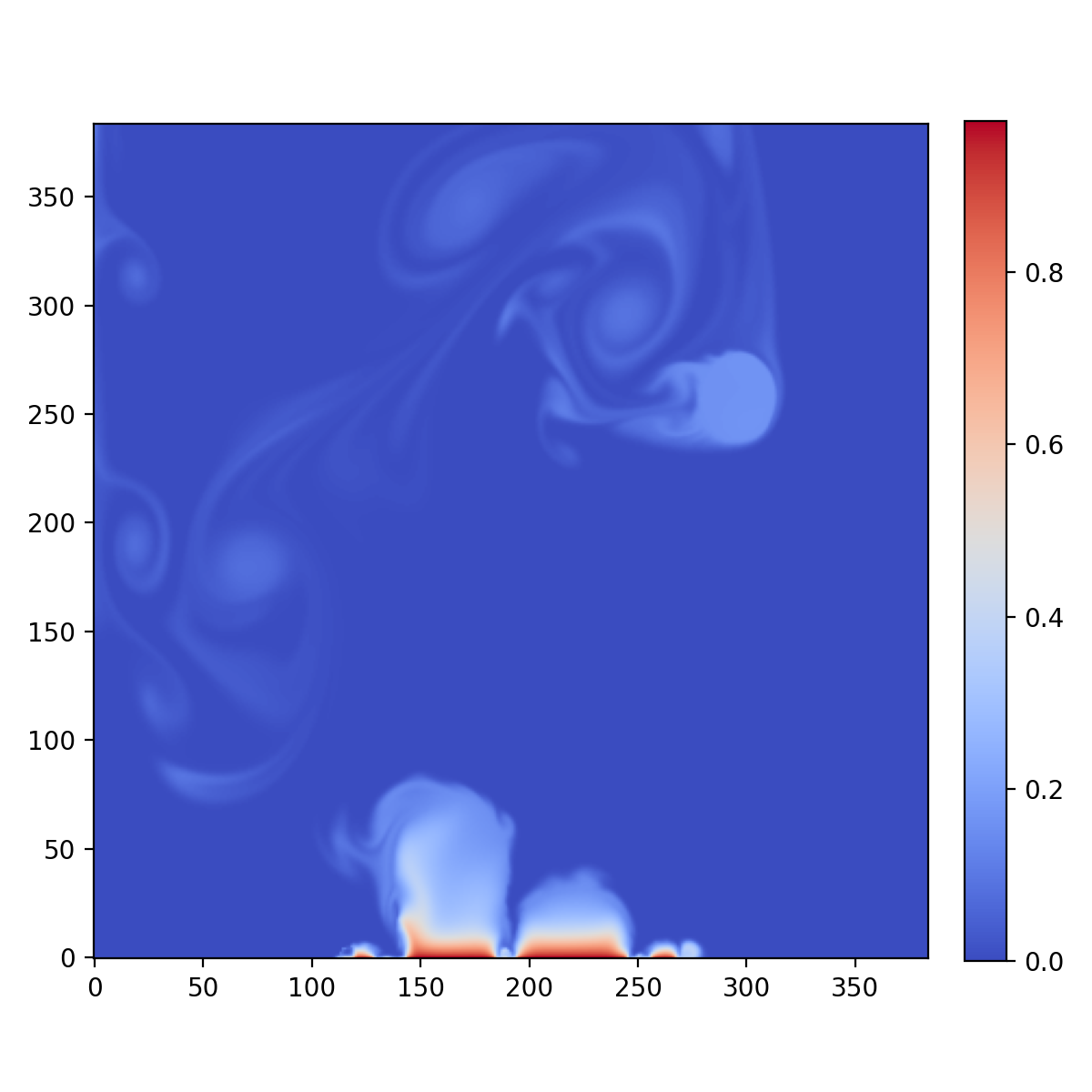}
        \hfill
    \includegraphics[width=0.35\linewidth]{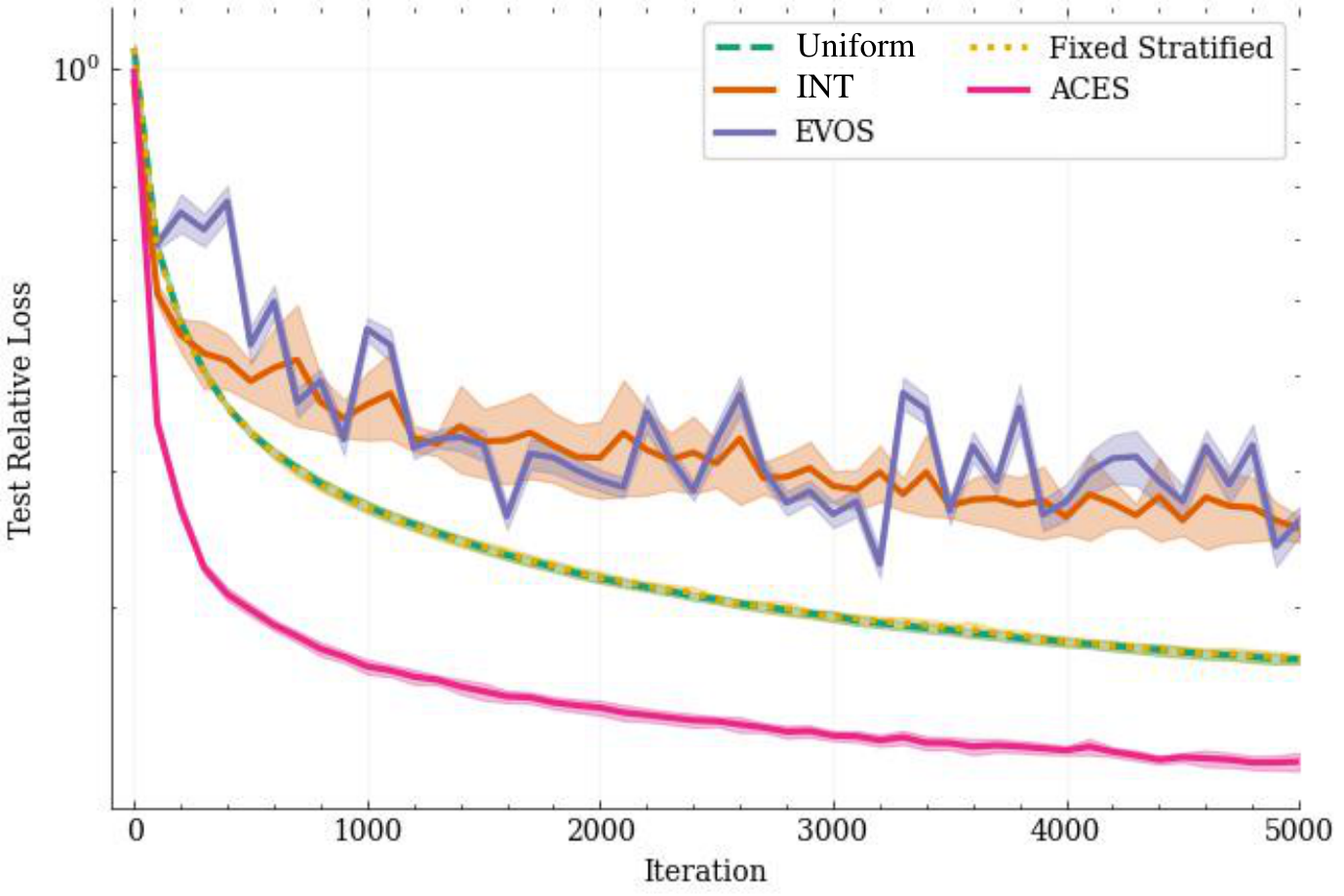}
    \hfill
    \includegraphics[width=0.35\linewidth]{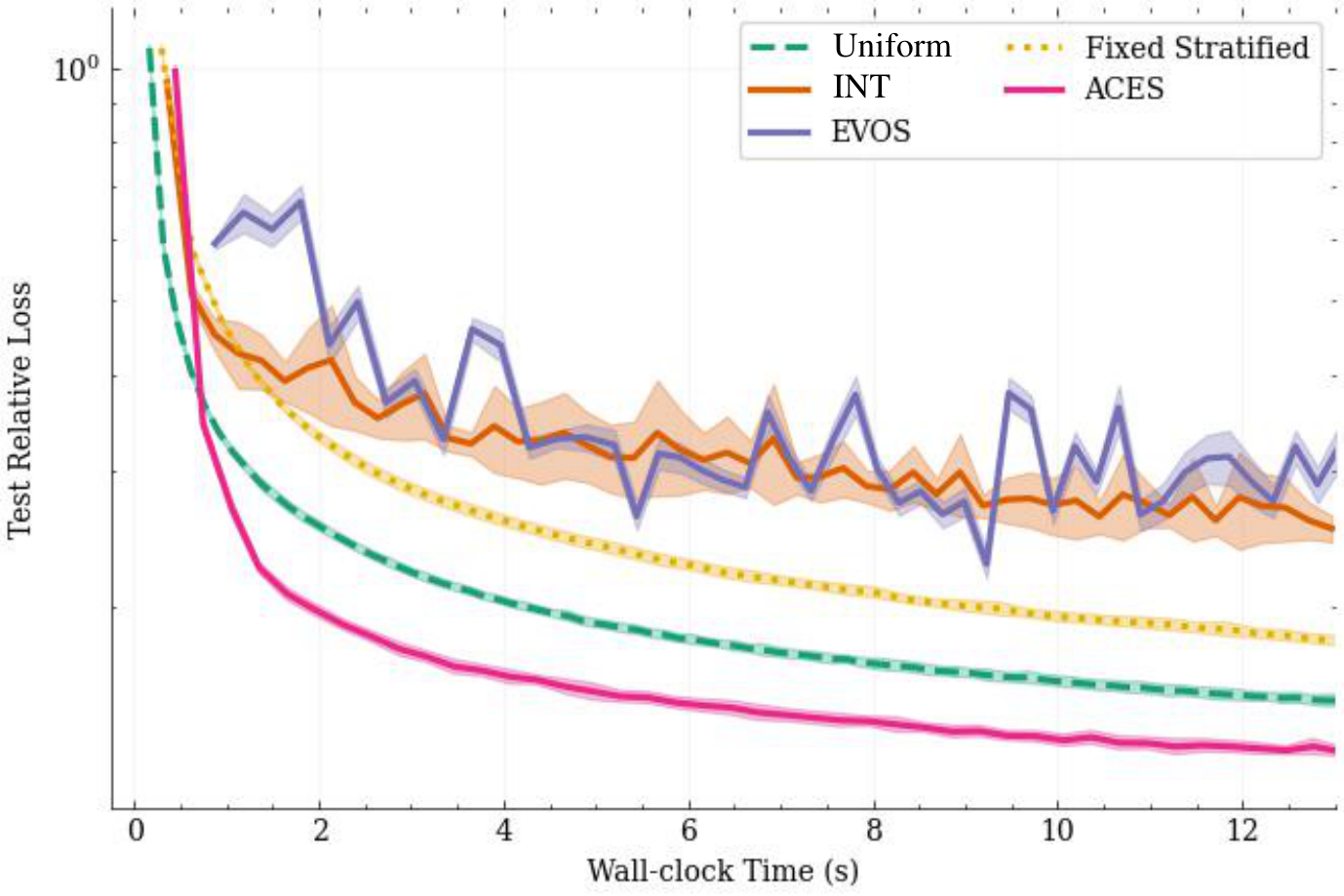}
    \caption{PoolBoiling2D: localized interfaces}
    \label{fig:pool_efficiency}
  \end{subfigure}

  \caption{
  \textbf{Training efficiency comparison.} Left: one representative snapshot per dataset. Middle: full-grid evaluation relative MSE versus iteration. Right: full-grid evaluation relative MSE versus runtime. ACES achieves substantially faster convergence per iteration than Uniform, INT, and EVOS on both datasets. Under equal runtime budgets, ACES remains competitive on NS2D and achieves the largest gains on PoolBoiling2D, where informative regions are highly localized.
  }
  \label{fig:exp_sampler_efficiency_combined}
\end{figure*}

\clearpage
\begin{table*}[t]
\centering
\caption{\textbf{Joint spatiotemporal reconstruction on PoolBoiling3D and NS3D.} All methods are evaluated at matched training checkpoints over 64 consecutive frames. Relative MSE (lower is better) is reported as mean $\pm$ standard deviation over three random seeds. Within each dataset, bold denotes the lowest mean error at each checkpoint. INT results are reported through 1,200 epochs.}
\label{tab:spatiotemporal_sampling_comparison}
\small
\setlength{\tabcolsep}{4pt}
\begin{tabular}{@{}lcccc@{}}
\toprule
\textbf{Method} & \textbf{200 epochs} & \textbf{1,200 epochs} & \textbf{2,400 epochs} & \textbf{5,000 epochs} \\
\midrule
\multicolumn{5}{@{}l}{\textit{PoolBoiling3D}} \\
INT                       & $0.8378 \pm 0.1320$ & $0.2600 \pm 0.0026$ & --- & --- \\
Uniform                   & $0.2948 \pm 0.0012$ & $0.1396 \pm 0.0014$ & $0.1133 \pm 0.0004$ & $0.0914 \pm 0.0016$ \\
Fixed Stratified          & $0.2760 \pm 0.0042$ & $0.1359 \pm 0.0085$ & $0.1231 \pm 0.0099$ & $0.0914 \pm 0.0043$ \\
ACES (partition)          & $\mathbf{0.2484 \pm 0.0073}$ & $\mathbf{0.1181 \pm 0.0010}$ & $\mathbf{0.0941 \pm 0.0011}$ & $\mathbf{0.0718 \pm 0.0007}$ \\
\midrule
\multicolumn{5}{@{}l}{\textit{NS3D}} \\
INT                       & $0.5378 \pm 0.0013$ & $0.3448 \pm 0.0110$ & --- & --- \\
Uniform                   & $0.3170 \pm 0.0004$ & $0.2120 \pm 0.0017$ & $0.1826 \pm 0.0017$ & $0.1572 \pm 0.0018$ \\
Fixed Stratified          & $\mathbf{0.3163 \pm 0.0008}$ & $0.2132 \pm 0.0008$ & $0.1837 \pm 0.0007$ & $0.1584 \pm 0.0008$ \\
ACES (partition)          & $0.3171 \pm 0.0020$ & $\mathbf{0.2049 \pm 0.0018}$ & $\mathbf{0.1763 \pm 0.0021}$ & $\mathbf{0.1520 \pm 0.0016}$ \\
\bottomrule
\end{tabular}
\end{table*}

\paragraph{Joint spatiotemporal reconstruction.}
Table~\ref{tab:spatiotemporal_sampling_comparison} reports reconstruction error on PoolBoiling3D and NS3D across matched training checkpoints (200--5,000 epochs), evaluated identically for both datasets. In this higher-dimensional setting, we evaluate the adaptive-partitioning component of ACES without residual-powered loss weighting. It achieves the lowest mean relative MSE at nearly every checkpoint, reducing error relative to uniform sampling by 21.4\% and 3.3\%, respectively, by 5,000 epochs. The sole exception is NS3D's earliest checkpoint (200 epochs), where Fixed Stratified is marginally better; adaptive partitioning overtakes it from 1,200 epochs onward. Fixed Stratified otherwise performs similarly to uniform sampling throughout. These results show that the coverage component of ACES extends beyond individual snapshots.

\subsection{Ablation Study}
We perform ablation studies on the PoolBoiling2D dataset to evaluate the contributions of adaptive partitioning and region-level weighting through three variants: (1) Fixed + Weighting, which applies residual-aware region weighting on top of the fixed partition, (2) Partition Only (Unbiased), which uses adaptive partitioning without residual-aware weighting, and (3) Full ACES, which combines adaptive partitioning with residual-aware weighting.

Figure~\ref{fig:ablation_ns1024_components} shows that both adaptive partitioning and residual-aware weighting contribute to the final performance. Adaptive partitioning alone already improves convergence over Fixed + Weighting, indicating that structured coverage and reduced sampling redundancy are important for efficient optimization. Adding residual-aware weighting further improves convergence speed, suggesting that adaptive partitioning and focused weighting provide complementary benefits.

Appendix~\ref{sec:appendix_runtime_sensitivity} additionally studies sensitivity to the weighting exponent $\alpha$ and partition-update interval.

\begin{figure*}[t]
  \centering
  \begin{minipage}[t]{0.42\textwidth}
    \centering
    \includegraphics[width=0.95\linewidth]{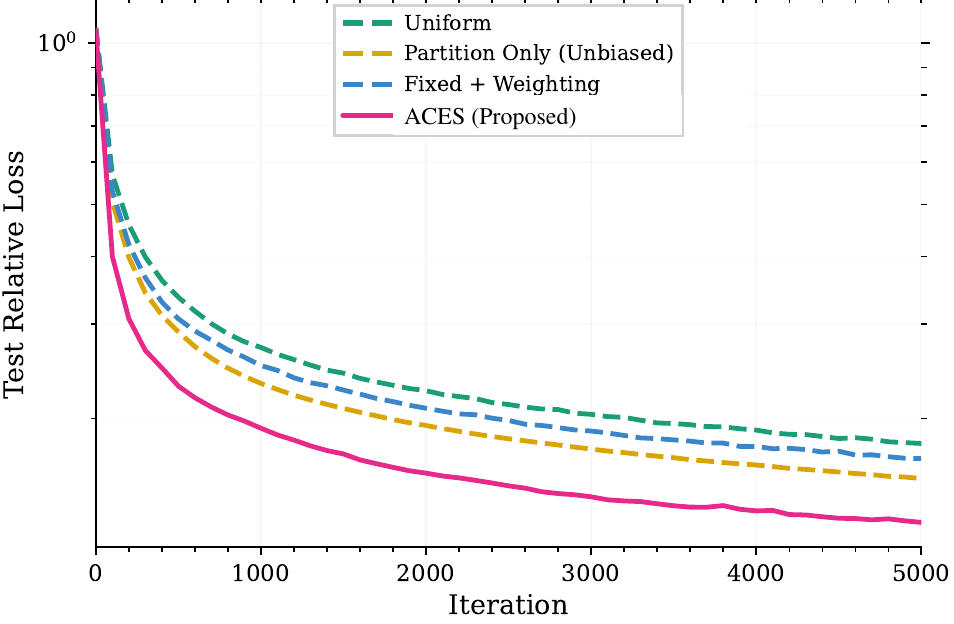}
    \captionof{figure}{\textbf{Component ablation.} We compare region-level weighting with a fixed partition, adaptive partitioning without weighting, and full ACES on PoolBoiling2D. Combining adaptive partitioning and residual-aware weighting yields the fastest convergence, indicating that the two components provide complementary benefits.}
    \label{fig:ablation_ns1024_components}
  \end{minipage}
  \hfill
  \begin{minipage}[t]{0.55\textwidth}
    \centering
    \includegraphics[width=0.95\linewidth]{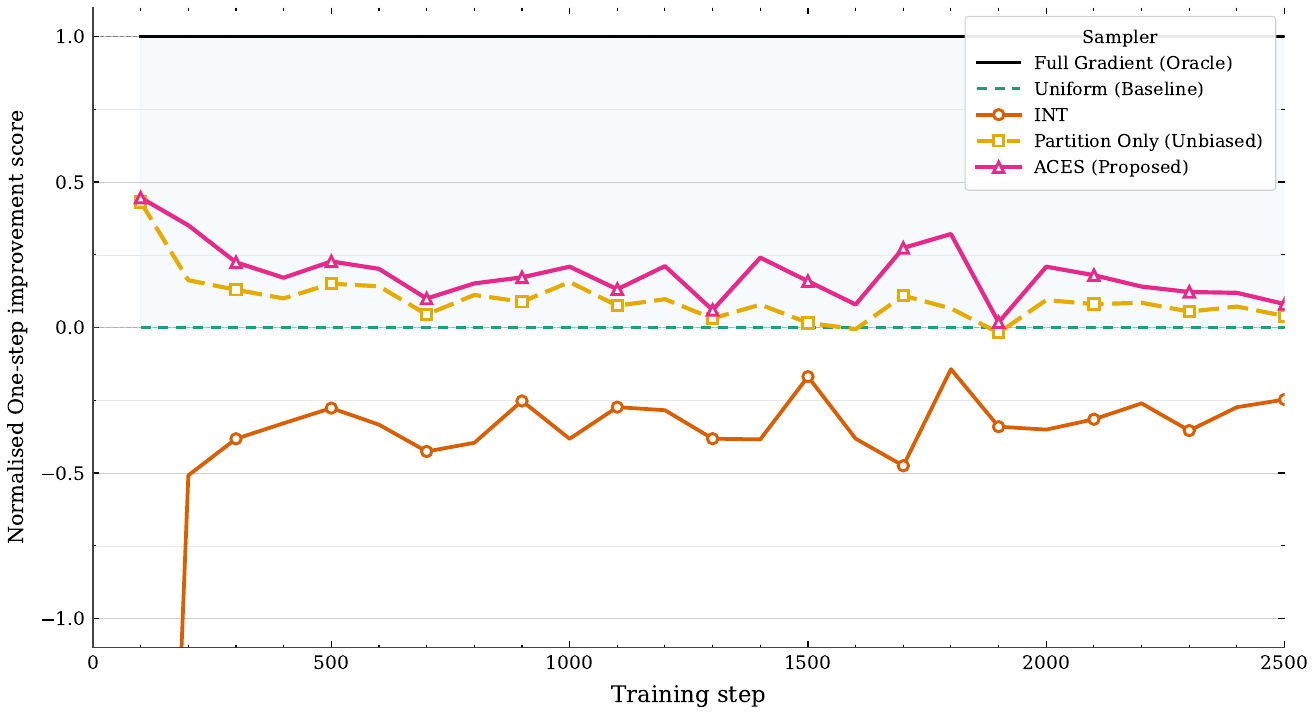}
    \caption{
        \textbf{Empirical Validation of the Optimization Proxy.}
        Normalized values of the one-step improvement proxy
        $S_w^2/G_w$ across training iterations. Higher values indicate better predicted optimization improvement. ACES achieves consistently higher scores than unbiased partitioning and uniform sampling, consistent with its faster empirical convergence.
        }
    \label{fig:grad_alignment_third_exp}
  \end{minipage}
\vspace{-2mm}
\end{figure*}


\subsection{Empirical Correlation with Optimization Behavior}
To evaluate whether the proposed optimization proxy reflects practical training behavior, we compare different sampling strategies using the bound-derived quantity $S_w^2/G_w$ from Section~\ref{sec:bias_variance_tradeoff}, where $S_w$ measures alignment with the full gradient and $G_w$ captures the estimator energy. We report a normalized version of this quantity using uniform sampling and the full-gradient oracle as reference points, where higher values indicate better predicted one-step optimization improvement.

To ensure a controlled comparison, all samplers are evaluated on the same set of checkpoints obtained from uniform-sampling training. At each checkpoint, we estimate the stochastic gradient statistics using six independently sampled minibatches for each sampling strategy, and evaluate the corresponding proxy values using the full gradient.

Figure~\ref{fig:grad_alignment_third_exp} shows that the ordering induced by the proxy is consistent with empirical convergence: methods with higher proxy values generally converge faster and achieve lower final error. In particular, ACES consistently achieves higher scores than both the Partition Only (Unbiased) variant and uniform sampling, while Partition Only also outperforms uniform sampling. These results support the theoretical motivation of ACES, suggesting that the proposed proxy captures a meaningful aspect of optimization efficiency.

Additional analyses of the proxy quality, including rank-correlation and top-$K$ overlap evaluations, are provided in Appendix~\ref{sec:appendix_proxy_validity}.

\section{Conclusion}
We introduced ACES, a region-level adaptive sampling framework for efficient training of implicit neural representations on heterogeneous continuous fields. ACES combines adaptive spatial partitioning with residual-aware weighting to balance domain coverage and focused refinement. Our analysis shows that adaptive partitioning reduces stochastic gradient variance through increased within-region homogeneity, while biased region-level weighting can improve optimization efficiency. Experiments on scientific field learning tasks demonstrate faster convergence and improved sample efficiency over uniform and pointwise adaptive sampling baselines, particularly in settings with highly localized complexity. The current implementation incurs additional sampling overhead, and the analysis characterizes local one-step behavior rather than long-term nonconvex dynamics. These results suggest that structured region-level sampling is a promising direction for efficient neural field learning.

\paragraph{Code availability.}
The source code and reproduction scripts are publicly available at \url{https://github.com/guangzhao27/ACES-Adaptive-Coverage-aware-Efficient-Sampling}.

\section*{Acknowledgments and Disclosure of Funding}
\textbf{Funding.} This material is based upon work supported by the U.S. Department of Energy, Office of Science, Office of Advanced Scientific Computing Research and Office of Biological and Environmental Research, via the RAPIDS Institute under the Scientific Discovery through Advanced Computing (SciDAC) program and the project of Improving Projections of AMOC and its Collapse Through Advanced Simulations (ImPACTS), under Contract No. DE-SC0012704 and FWP No. CC122 and CC127, and award B\&R-KJ0401010 and FWP-CC147. \\
\textbf{Competing interests.} The authors declare no competing interests.

\bibliographystyle{plainnat}
\bibliography{reference}


\appendix

\section{Proof of Variance Reduction}
\label{app:variance_reduction}

We show that, under a local variogram model for the gradient field, refining a
$d$-dimensional cell uniformly into $m^d$ subcells reduces its contribution to the
sampling variance by a fixed multiplicative factor. This justifies choosing the
cell with the largest current variance contribution for refinement.

\paragraph{Setup.}
Let $\mathcal D_k$ be a hypercubic cell with side length $h_k$, containing $N_k$
points and allocated $n_k$ samples. Its contribution to the stratified sampling
variance is
\begin{equation}
\Delta V_k
=
\left(\frac{N_k}{N}\right)^2
\frac{\sigma_k^2}{n_k},
\end{equation}
where $\sigma_k^2$ denotes the within-cell variance of the gradient field.

Suppose $\mathcal D_k$ is split into $m^d$ congruent hypercubic subcells
$\{\mathcal D_{k,i}\}_{i=1}^{m^d}$. Under a uniform refinement, each subcell
contains
\begin{equation}
N_{k,i} = \frac{N_k}{m^d},
\qquad
n_{k,i} = \frac{n_k}{m^d},
\qquad i = 1, \dots, m^d,
\end{equation}
and has side length $h_k / m$.

\paragraph{Local variogram model.}
Assume that within cell $k$, the gradient field satisfies the local variogram
relation \citep{cressie2015statistics}
\begin{equation}
\frac{1}{2}\mathbb{E}\big[\|\mathbf g(x)-\mathbf g(y)\|^2\big]
=
a_k \|x-y\|^\nu,
\qquad 0 < \nu \le 2,
\end{equation}
where $a_k$ may depend on the cell, while the exponent $\nu$ is shared across
cells. Using the standard identity
\begin{equation}
\sigma_k^2
=
\frac{1}{2}
\mathbb{E}_{x,y \in \mathcal D_k}
\big[
\|\mathbf g(x)-\mathbf g(y)\|^2
\big],
\end{equation}
we obtain
\begin{equation}
\sigma_k^2
=
a_k
\mathbb{E}_{x,y \in \mathcal D_k}
\big[
\|x-y\|^\nu
\big].
\end{equation}
For a hypercubic cell of side length $h_k$, the expectation above scales as
$h_k^\nu$, so
\begin{equation}
\sigma_k^2 \propto a_k h_k^\nu.
\end{equation}

Similarly, each refined subcell $\mathcal D_{k,i}$ has side length $h_k / m$,
and therefore
\begin{equation}
\sigma_{k,i}^2
\propto
a_k \left(\frac{h_k}{m}\right)^\nu
=
m^{-\nu} \sigma_k^2.
\end{equation}

\paragraph{Variance after refinement.}
The total variance contribution of the refined cell is the sum over all
subcells:
\begin{equation}
\Delta V_k^{\mathrm{split}}
=
\sum_{i=1}^{m^d}
\left(\frac{N_{k,i}}{N}\right)^2
\frac{\sigma_{k,i}^2}{n_{k,i}}.
\end{equation}
Substituting $N_{k,i} = N_k / m^d$, $n_{k,i} = n_k / m^d$, and
$\sigma_{k,i}^2 = m^{-\nu} \sigma_k^2$ gives
\begin{align}
\Delta V_k^{\mathrm{split}}
&=
\sum_{i=1}^{m^d}
\left(\frac{N_k}{m^d N}\right)^2
\frac{m^{-\nu} \sigma_k^2}{n_k / m^d} \\
&=
\sum_{i=1}^{m^d}
m^{-\nu-d}
\left(\frac{N_k}{N}\right)^2
\frac{\sigma_k^2}{n_k} \\
&=
m^d \cdot m^{-\nu-d} \Delta V_k \\
&=
m^{-\nu} \Delta V_k.
\end{align}

Hence the variance reduction induced by refining cell $k$ is
\begin{equation}
\Delta V_k - \Delta V_k^{\mathrm{split}}
=
\left(1 - m^{-\nu}\right) \Delta V_k.
\end{equation}
In particular, the expected gain from refinement is proportional to the current
variance contribution $\Delta V_k$. Therefore, among candidate cells, the
maximally beneficial split is obtained by selecting the cell with the largest
$\Delta V_k$.

\section{Derivation of the Gradient Variance Approximation}
\label{app:grad_variance}

Starting from
\[
\mathbf g_i = \mathbf J_i^\top r_i,
\]
we assume that the residual $r_i$ is a scalar and that $\mathbf J_i$ is a vector-valued Jacobian within region $k$. We define the regional means as
\[
\bar r_k := \mathbb{E}[r_i \mid k], 
\qquad
\bar{\mathbf J}_k := \mathbb{E}[\mathbf J_i \mid k].
\]
Let the centered variables be
\[
\delta r_i := r_i - \bar r_k,
\qquad
\delta \mathbf J_i := \mathbf J_i - \bar{\mathbf J}_k.
\]
Then
\[
\mathbf g_i
= (\bar{\mathbf J}_k + \delta \mathbf J_i)^\top (\bar r_k + \delta r_i).
\]
Expanding this expression gives
\[
\mathbf g_i
=
\bar{\mathbf J}_k^\top \bar r_k
+
\bar{\mathbf J}_k^\top \delta r_i
+
\delta \mathbf J_i^\top \bar r_k
+
\delta \mathbf J_i^\top \delta r_i.
\]
Using a first-order approximation, we neglect the second-order term
$\delta \mathbf J_i^\top \delta r_i$, so that
\[
\mathbf g_i - \bar{\mathbf g}_k
\approx
\bar{\mathbf J}_k^\top \delta r_i
+
\bar r_k\,\delta \mathbf J_i^\top,
\]
where
\[
\bar{\mathbf g}_k := \mathbb{E}[\mathbf g_i \mid k]
\approx
\bar{\mathbf J}_k^\top \bar r_k.
\]

The total variance of $\mathbf g_i$ in region $k$ is defined as
\[
\sigma_{g,k}^2
:=
\operatorname{tr}\!\big(\operatorname{Cov}(\mathbf g_i \mid k)\big)
=
\mathbb{E}\!\left[\|\mathbf g_i - \bar{\mathbf g}_k\|_2^2 \mid k\right].
\]
Substituting the first-order approximation yields
\begin{align}
\sigma_{g,k}^2
&\approx
\mathbb{E}\!\left[
\left\|
\bar{\mathbf J}_k^\top \delta r_i
+
\bar r_k\,\delta \mathbf J_i^\top
\right\|_2^2
\middle| k
\right] \\
&=
\mathbb{E}\!\left[\|\bar{\mathbf J}_k^\top \delta r_i\|_2^2 \mid k\right]
+
\mathbb{E}\!\left[\|\bar r_k\,\delta \mathbf J_i^\top\|_2^2 \mid k\right]
+
2\,\mathbb{E}\!\left[
\left\langle
\bar{\mathbf J}_k^\top \delta r_i,
\bar r_k\,\delta \mathbf J_i^\top
\right\rangle
\middle| k
\right].
\end{align}
If we neglect the cross term, then
\[
\sigma_{g,k}^2
\approx
\|\bar{\mathbf J}_k\|_2^2\,\mathbb{E}[\delta r_i^2 \mid k]
+
\bar r_k^{\,2}\,\mathbb{E}\!\left[\|\delta \mathbf J_i\|_2^2 \mid k\right].
\]
Finally, defining
\[
\sigma_{r,k}^2 := \operatorname{Var}(r_i \mid k)
= \mathbb{E}[\delta r_i^2 \mid k],
\qquad
\sigma_{J,k}^2 := \operatorname{tr}\!\big(\operatorname{Cov}(\mathbf J_i \mid k)\big)
= \mathbb{E}\!\left[\|\delta \mathbf J_i\|_2^2 \mid k\right],
\]
we obtain
\[
\sigma_{g,k}^2
\approx
\|\bar{\mathbf J}_k\|_2^2\,\sigma_{r,k}^2
+
\bar r_k^{\,2}\,\sigma_{J,k}^2.
\]

\section{Local Improvement Along the Residual-Power Family}

\subsection{A Sufficient Condition for Local Improvement with Residual-Power Weighting}
\label{app:alpha_monotone}

We analyze the residual-power weighting family used by ACES directly. Fix a partition with region
proportions $p_k$, regional mean gradients $\bar{\mathbf g}_k$, and independent regional minibatch
estimates $\hat{\mathbf g}_k$. Let
\[
\mathbf g=\sum_{k=1}^{K}p_k\bar{\mathbf g}_k,
\qquad
c_k=p_k^2\,\mathbb E\|\hat{\mathbf g}_k-\bar{\mathbf g}_k\|^2,
\qquad
V=\sum_{k=1}^{K}c_k.
\]
For normalized weights satisfying $\sum_k p_k w_k=1$, define
\[
\bar{\mathbf g}_w=\sum_{k=1}^{K}p_kw_k\bar{\mathbf g}_k,
\qquad
F(w)=
\frac{\langle\mathbf g,\bar{\mathbf g}_w\rangle^2}
{\|\bar{\mathbf g}_w\|^2+\sum_{k=1}^{K}w_k^2c_k}.
\]
This is the one-step proxy $S_w^2/G_w$ from Eq.~\eqref{eq:alignment_energy_ratio}. We further define
\[
a_k=p_k\langle\mathbf g,\bar{\mathbf g}_k\rangle,
\qquad
t_k=a_kV-\|\mathbf g\|^2c_k.
\]
Because $\sum_k a_k=\|\mathbf g\|^2$ and $\sum_k c_k=V$, we have $\sum_k t_k=0$.

The residual-power family in Eq.~\eqref{eq:practical_region_weight} is
\[
w_k(\alpha)
=
\frac{|\bar r_k|^\alpha}{\sum_{j=1}^{K} p_j |\bar r_j|^\alpha},
\qquad
\alpha \geq 0,
\]
which satisfies $\sum_k p_k w_k(\alpha)=1$ and recovers unbiased weighting at $\alpha=0$.

\begin{proposition}[Local improvement under residual-score correlation]
\label{prop:app_alpha_monotone}
Assume $\mathbf g\neq\mathbf 0$, $c_k>0$, and $|\bar r_k|>0$ for all $k$. Let
$\ell_k=\log|\bar r_k|$. Then
\[
\left.
\frac{d}{d\alpha} F\bigl(w(\alpha)\bigr)
\right|_{\alpha = 0}
=
\frac{2\|\mathbf g\|^2}{(\|\mathbf g\|^2+V)^2}
\sum_{k=1}^{K} t_k\, \ell_k.
\]
In particular, if the region improvement scores and log-residual magnitudes are positively
correlated in the sense that
\[
\sum_{k=1}^{K} t_k \log|\bar r_k| \;>\; 0,
\]
then there exists $\bar\alpha > 0$ such that $F(w(\alpha)) > F(w(0)) = F(\mathbf 1)$ for all
$\alpha \in (0, \bar\alpha)$: a small residual-power bias strictly improves the one-step proxy
over the unbiased stratified estimator.
\end{proposition}

\begin{proof}
Write $w_k(\alpha) = e^{\alpha \ell_k} / Z(\alpha)$ with $Z(\alpha) = \sum_j p_j e^{\alpha
\ell_j}$, so $Z(0) = 1$ and $Z'(0) = \sum_j p_j \ell_j =: \bar\ell$. Differentiating,
\[
\left.
\frac{d w_k}{d\alpha}
\right|_{\alpha = 0}
=
\left.
\frac{\ell_k e^{\alpha \ell_k} Z(\alpha) - e^{\alpha \ell_k} Z'(\alpha)}{Z(\alpha)^2}
\right|_{\alpha = 0}
=
\ell_k - \bar\ell.
\]
At $w=\mathbf 1$, the numerator components of $F$ satisfy
\[
S_{\mathbf 1}=\|\mathbf g\|^2,
\qquad
G_{\mathbf 1}=\|\mathbf g\|^2+V,
\qquad
\left.\frac{\partial S_w}{\partial w_k}\right|_{\mathbf 1}=a_k,
\qquad
\left.\frac{\partial G_w}{\partial w_k}\right|_{\mathbf 1}=2(a_k+c_k).
\]
The quotient rule therefore gives
\[
\left.\frac{\partial F}{\partial w_k}\right|_{\mathbf 1}
=
\frac{2\|\mathbf g\|^2}{(\|\mathbf g\|^2+V)^2}
\left(a_kV-\|\mathbf g\|^2c_k\right)
=
\frac{2\|\mathbf g\|^2}{(\|\mathbf g\|^2+V)^2}t_k.
\]
Applying the chain rule,
\[
\left.
\frac{d}{d\alpha} F\bigl(w(\alpha)\bigr)
\right|_{\alpha=0}
=
\sum_{k=1}^{K}
\left.
\frac{\partial F}{\partial w_k}
\right|_{\mathbf 1}
(\ell_k - \bar\ell)
=
\frac{2\|\mathbf g\|^2}{(\|\mathbf g\|^2+V)^2}
\sum_{k=1}^{K} t_k (\ell_k - \bar\ell).
\]
Since $\sum_{k=1}^{K} t_k = 0$, the term $\bar\ell \sum_k t_k$ vanishes, yielding the stated
derivative. If the sufficient condition holds, this derivative is strictly positive, and
continuity of $\alpha \mapsto F(w(\alpha))$ gives the claimed strict improvement on an interval
$(0, \bar\alpha)$.
\end{proof}


\subsection{Empirical Verification of the Ascent Condition}
\label{app:alpha_monotone_empirical}

We test the sufficient condition of Proposition~\ref{prop:app_alpha_monotone},
$\sum_{k=1}^{K} t_k \log|\bar r_k| > 0$, directly on trained checkpoints. For each of
PoolBoiling2D and NS2D, we take five independently trained ACES runs (one per time frame) and
freeze the model at six checkpoints spanning training (epochs 100, 300, 1,000, 2,000, 4,000, and 4,999).

Across all $2 \text{ datasets} \times 5 \text{ runs} \times 6 \text{ checkpoints} = 60$
configurations, the identity $\sum_k t_k = 0$ holds to within $3\times10^{-16}$ relative error (a
correctness check on the construction), and the sufficient condition
$\sum_k t_k \log|\bar r_k| > 0$ holds in \emph{every} configuration: $60/60$. The
underlying Spearman rank correlation $\rho(t_k, \log|\bar r_k|)$ stays positive
throughout training on both datasets, ranging over $[0.225, 0.606]$ on PoolBoiling2D and
$[0.312, 0.774]$ on NS2D. Figure~\ref{fig:alpha_monotone_empirical} plots, for both datasets and
across training, (a) the resulting ascent slope $\left.dF/d\alpha\right|_{\alpha=0}$ from
Proposition~\ref{prop:app_alpha_monotone} (log-scale, since it spans several orders of magnitude
between the two datasets, and is guaranteed positive by the result above) and (b) the underlying
rank correlation. Together with Appendix~\ref{sec:appendix_proxy_validity}, this confirms that the
first-order condition licensing a small residual-power bias holds throughout training, on both
evaluated datasets, not merely at initialization.

\begin{figure}[t]
  \centering
  \includegraphics[width=0.85\linewidth]{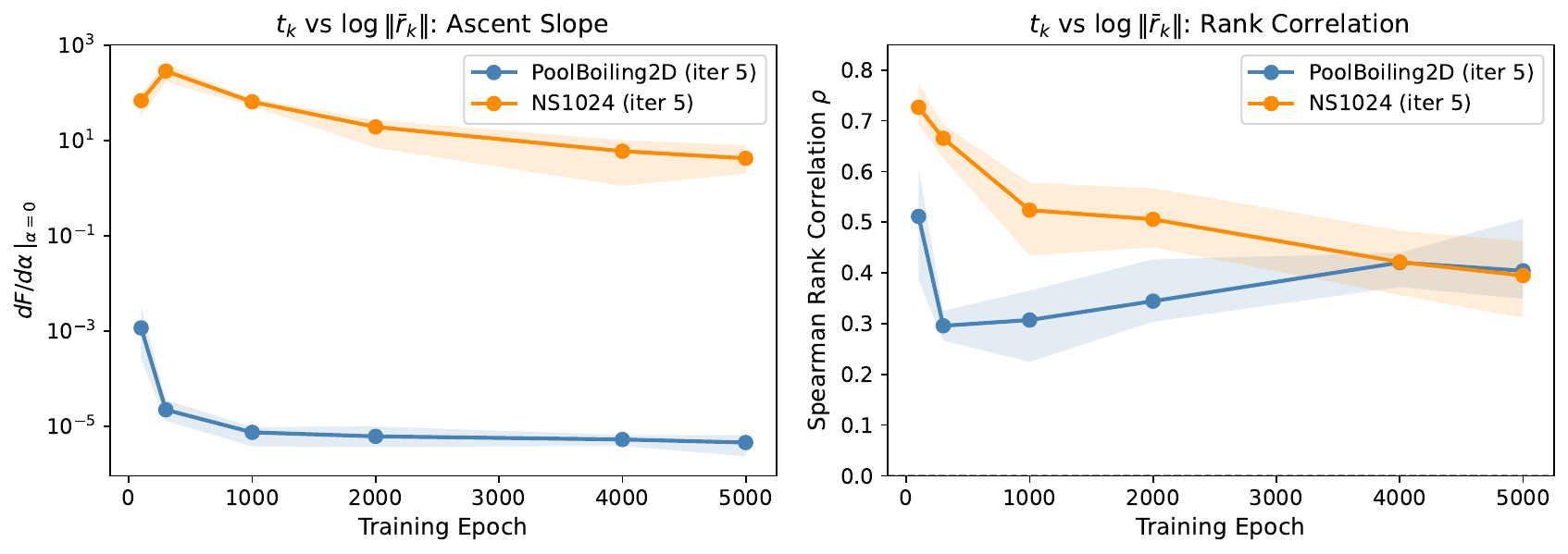}
  \caption{Empirical verification of the sufficient condition in
  Proposition~\ref{prop:app_alpha_monotone} on PoolBoiling2D and NS2D using a budget-matched ACES
  partition ($\alpha=0$, i.e.\ $w=\mathbf 1$). The curves show the mean with the min--max range over five time-frame runs per
  dataset. \textbf{(a)} The ascent slope $\left.dF/d\alpha\right|_{\alpha=0} = \frac{2\|\mathbf g\|^2}
  {(\|\mathbf g\|^2+V)^2}\sum_k t_k \log|\bar r_k|$ stays strictly positive at every training
  checkpoint on both datasets ($60/60$ configurations). \textbf{(b)} The underlying Spearman rank
  correlation $\rho(t_k, \log|\bar r_k|)$ also stays positive throughout training,
  confirming that regions with larger residual magnitude tend to have a more favorable
  alignment-to-variance ratio $t_k$, not only near initialization.}
  \label{fig:alpha_monotone_empirical}
\end{figure}

\section{Original One-Step Improvement Score}
\label{sec:appendix_original_one_step_score}

In the main text, we affinely normalize the proxy so that uniform sampling has value $0$ and the full-gradient oracle has value $1$ at each checkpoint. For reference, Figure~\ref{fig:appendix_grad_alignment_original_metric} shows the original one-step improvement score corresponding directly to $S_w^2/G_w$.

\begin{figure}[t]
	\centering
	\includegraphics[width=0.7\linewidth]{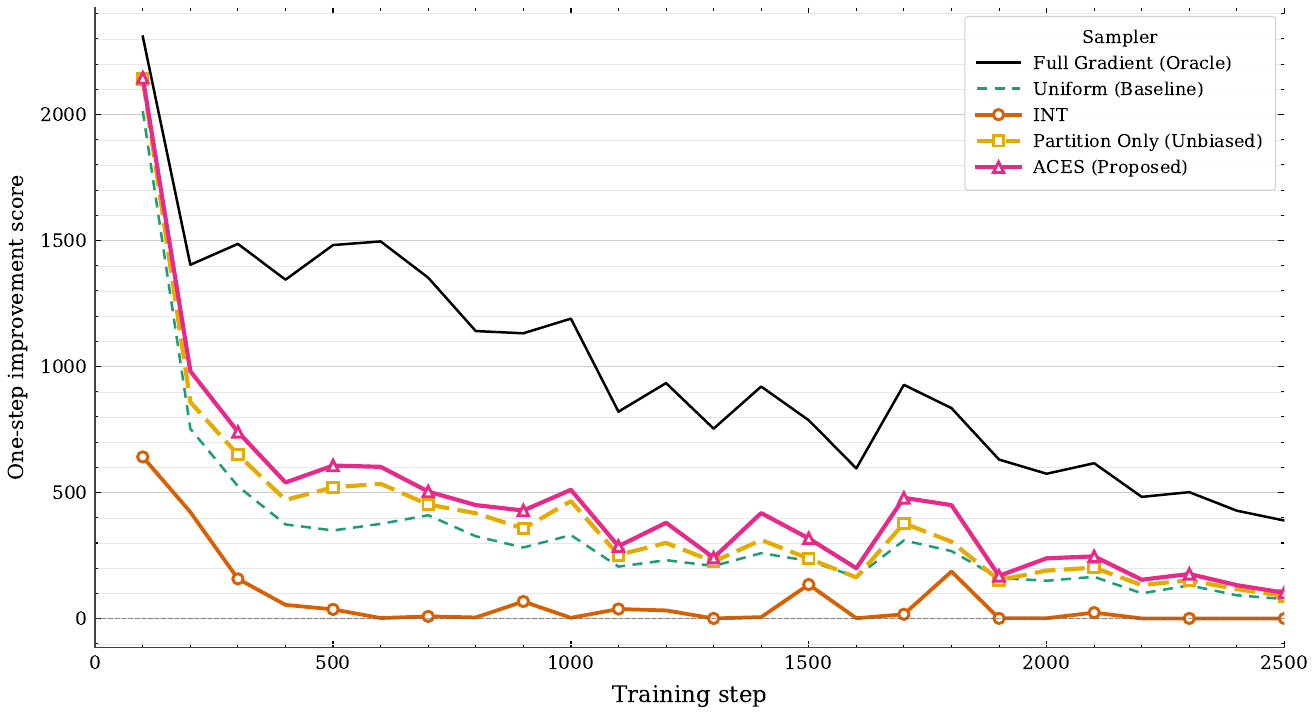}
	\caption{Original one-step improvement score across training steps, corresponding to $S_w^2 / G_w$ before normalization/sign transformation.}
	\label{fig:appendix_grad_alignment_original_metric}
\end{figure}

\section{Validity of Residual Variance as a Proxy}
\label{sec:appendix_proxy_validity}

To justify our proxy choice, we compare residual variance $\sigma_{r,k}^2$ with gradient variance $\sigma_{g,k}^2$ at the region level on NS2D using checkpoints from uniform sampling. We evaluate two cell granularities ($16\times16$ and $32\times32$ regions) and report two diagnostics over training: (1) Spearman rank correlation between region scores, and (2) top-$K$ overlap between the highest-scoring regions under $\sigma_{g,k}^2$ and $\sigma_{r,k}^2$.

Figure~\ref{fig:proxy_rank_topk} shows strong agreement between residual variance and gradient variance throughout training. The mean Spearman correlation remains high for both granularities, reaching $\rho=0.807\pm0.031$ ($16\times16$) and $\rho=0.820\pm0.054$ ($32\times32$). In addition, top-$K$ overlap remains substantially above the random baseline across all evaluated thresholds. For example, at top-20\%, the overlap reaches $0.719\pm0.059$ ($16\times16$) and $0.660\pm0.070$ ($32\times32$), compared to the random baseline of $0.20$. These results indicate that residual variance largely identifies the same informative regions as direct gradient variance while being much cheaper to estimate, supporting its use as a scalable proxy in ACES.

\begin{figure}[t]
  \centering
  \includegraphics[width=0.8\linewidth]{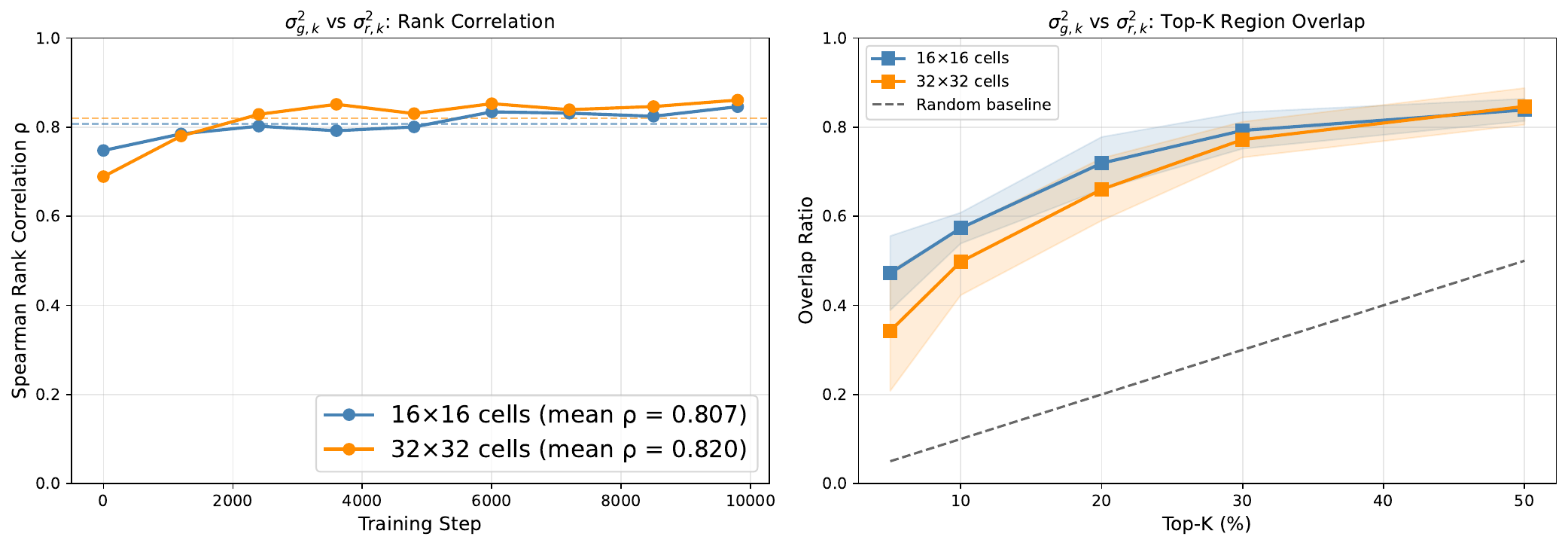}
  \caption{Validation of residual variance as a proxy for gradient variance on NS2D. Left: Spearman rank correlation between $\sigma_{g,k}^2$ and $\sigma_{r,k}^2$ across checkpoints for random $16\times16$ and $32\times32$ region sets. Right: top-$K$ overlap between regions selected by $\sigma_{g,k}^2$ and $\sigma_{r,k}^2$. Both curves stay well above the random baseline, showing that residual variance preserves region ranking and high-importance region retrieval.}
  \label{fig:proxy_rank_topk}
\end{figure}

\section{Generalization to a Fourier Feature Network}
\label{sec:appendix_ffn_backbone}

To test whether the benefits of ACES depend on the SIREN architecture used in the main experiments, we replace SIREN with a Fourier Feature Network (FFN)~\citep{tancik2020fourierfeatures} and repeat the PoolBoiling2D comparison. All sampling methods use the same FFN backbone and evaluation protocol. Table~\ref{tab:appendix_ffn_backbone} reports relative MSE throughout training, averaged over three random seeds.

\begin{table}[t]
  \centering
  \scriptsize
  \setlength{\tabcolsep}{2pt}
  \caption{Relative MSE on PoolBoiling2D using a Fourier Feature Network backbone (mean $\pm$ standard deviation over three seeds). Lower is better; bold denotes the lowest mean at each checkpoint.}
  \label{tab:appendix_ffn_backbone}
  \begin{tabular}{@{}lccccc@{}}
    \toprule
    \textbf{Method} & \textbf{200} & \textbf{400} & \textbf{800} & \textbf{2,000} & \textbf{5,000} \\
    \midrule
    Uniform random & $\mathbf{0.9537\pm0.0047}$ & $0.9435\pm0.0141$ & $0.8656\pm0.0730$ & $0.5759\pm0.1898$ & $0.2866\pm0.0411$ \\
    INT & $1.1370\pm0.0818$ & $\mathbf{0.9045\pm0.0573}$ & $0.7092\pm0.0425$ & $0.5009\pm0.0324$ & $0.3359\pm0.0231$ \\
    EVOS & $2.6540\pm0.2279$ & $1.1808\pm0.0780$ & $0.7930\pm0.0544$ & $0.4755\pm0.0292$ & $0.3363\pm0.0189$ \\
    Fixed stratified & $0.9564\pm0.0014$ & $0.9483\pm0.0093$ & $0.9050\pm0.0366$ & $0.6387\pm0.1542$ & $0.3262\pm0.0590$ \\
    ACES & $1.4906\pm0.1301$ & $1.1236\pm0.4569$ & $\mathbf{0.6154\pm0.3184}$ & $\mathbf{0.2453\pm0.0469}$ & $\mathbf{0.1536\pm0.0263}$ \\
    \bottomrule
  \end{tabular}
\end{table}

Uniform sampling and INT obtain lower mean error during the earliest checkpoints, whereas ACES obtains the lowest mean error from 800 iterations onward. At 5,000 iterations, ACES reaches a relative MSE of $0.1536\pm0.0263$, compared with $0.2866\pm0.0411$ for uniform sampling. This result indicates that the later-stage advantage of adaptive region-level sampling is not specific to the SIREN backbone, while also showing that ACES can have a slower initial phase with this architecture.

\section{Experimental Settings and Hardware}
\label{sec:appendix_experiment_setting_hardware}

This section summarizes the practical training setup used for the experiments in this paper.

\paragraph{Common training setup.}
Unless otherwise stated, we use a SIREN INR backbone. The spatial experiments draw a fraction $2\times10^{-3}$ of the candidate points per iteration, and the joint spatiotemporal experiments draw $7\times10^{-4}$. The FFN experiment in Appendix~\ref{sec:appendix_ffn_backbone} changes only the backbone comparison described there. We report relative MSE and average results over three random seeds.

\paragraph{Experiment configuration.}
All SIREN experiments use SGD with momentum $0.9$, Nesterov acceleration, zero weight decay, and zero dampening. The spatial experiments use depth $3$, hidden width $155$, and $w_0=30$; their learning rates are $10^{-5}$ for NS2D and $10^{-3}$ for PoolBoiling2D. The joint spatiotemporal experiments use depth $10$, hidden width $384$, and $w_0=60$. Their learning rates are $2\times10^{-3}$ on NS3D and $5\times10^{-3}$ on PoolBoiling3D for Uniform, Fixed Stratified, and ACES, and $10^{-4}$ for INT. Learning rates are selected by a coarse search over values from $10^{-6}$ to $10^{-2}$. In the joint experiments, ACES uses adaptive partitioning without residual-powered loss weighting; this configuration isolates whether adaptive coverage scales to the $(x,y,t)$ domain.

\paragraph{Hardware and software environment.}
All reported runs were executed on a single NVIDIA A100-SXM4-80GB GPU (80 GB memory). The runtime environment reported by \texttt{nvidia-smi} is: Driver Version \texttt{580.105.08}, CUDA Version \texttt{13.0}, with GPU model \texttt{NVIDIA A100-SXM4-80GB}. This hardware specification is used for runtime and wall-clock comparisons in the paper.

\section{Runtime and Hyperparameter Sensitivity}
\label{sec:appendix_runtime_sensitivity}

\subsection{Runtime and Overhead Analysis}
We re-profile the current NS2D configuration on a single NVIDIA A100-SXM4-80GB GPU. For each of three random seeds, we first run 20 warm-up iterations and then measure 1,000 training iterations. We synchronize CUDA before and after every timed region and separately measure the model-training stage, ordinary sampling, and adaptive partition construction. The model-training stage includes the weighted loss, forward and backward passes, gradient processing, and optimizer update. ACES rebuilds its partition every 200 iterations; therefore, each measured run contains five partition updates. Table~\ref{tab:appendix_runtime_breakdown} reports the mean and standard deviation across the three runs.

\begin{table}[t]
  \centering
  \small
  \setlength{\tabcolsep}{3.5pt}
  \caption{Runtime breakdown on NS2D in milliseconds per training iteration (mean $\pm$ standard deviation over three seeds). Partition construction is excluded from the sampling column and reported separately as an amortized cost. The final column gives the cost of one partition update.}
  \label{tab:appendix_runtime_breakdown}
  \begin{tabular}{@{}lccccc@{}}
    \toprule
    \textbf{Method} & \textbf{Training} & \textbf{Sampling} & \textbf{Partition} & \textbf{Total} & \textbf{Per update} \\
    \midrule
    Uniform & $1.2474\pm0.0002$ & $0.4140\pm0.0008$ & --- & $1.6615\pm0.0010$ & --- \\
    ACES & $1.4688\pm0.0051$ & $1.3547\pm0.0051$ & $0.1609\pm0.0007$ & $2.9844\pm0.0037$ & $32.184\pm0.144$ \\
    \bottomrule
  \end{tabular}
\end{table}

ACES costs approximately $1.80\times$ as much per iteration as uniform sampling. Most of the difference comes from region-wise sample allocation and construction of the weighted minibatch; periodic partition construction contributes only $0.161$ ms per iteration after amortization. The larger training-stage cost reflects the weighted loss operations.

\subsection{Sensitivity to Weighting and Partition Updates}
We first vary the residual weighting exponent $\alpha$ while keeping the remaining sampler settings fixed. Table~\ref{tab:appendix_alpha_sensitivity} reports relative MSE averaged over five frames. Increasing $\alpha$ from a weak weighting regime improves reconstruction on both datasets, while the improvement saturates in a moderate range. Very strong concentration provides no further consistent gain.

\begin{table}[t]
  \centering
  \small
  \caption{Sensitivity to the weighting exponent $\alpha$. Lower relative MSE is better; bold denotes the lowest value in each row.}
  \label{tab:appendix_alpha_sensitivity}
  \begin{tabular}{@{}lcccccccc@{}}
    \toprule
    $\alpha$ & 0.1 & 0.2 & 0.5 & 1.0 & 2.0 & 3.0 & 4.0 & 5.0 \\
    \midrule
    PoolBoiling2D & 0.1425 & 0.1363 & 0.1258 & 0.1097 & \textbf{0.1084} & 0.1096 & 0.1123 & 0.1117 \\
    NS2D & 0.2816 & 0.2744 & 0.2605 & 0.2367 & 0.2196 & 0.2163 & \textbf{0.2156} & 0.2163 \\
    \bottomrule
  \end{tabular}
\end{table}

We next vary the number of iterations between partition updates. As shown in Table~\ref{tab:appendix_interval_sensitivity}, performance changes little for intervals between 50 and 1,000 iterations. In contrast, constructing the partition only at initialization (the 5,000-iteration setting) increases error on both datasets. These results indicate that periodic adaptation matters, while performance is not sensitive to the exact update frequency over a broad range.

\begin{table}[t]
  \centering
  \small
  \caption{Sensitivity to the partition-update interval. Relative MSE is reported as mean $\pm$ standard deviation; lower is better. ``Initial only'' constructs the partition at initialization and does not refresh it during the 5,000-iteration run.}
  \label{tab:appendix_interval_sensitivity}
  \begin{tabular}{@{}lcc@{}}
    \toprule
    \textbf{Update interval} & \textbf{PoolBoiling2D} & \textbf{NS2D} \\
    \midrule
    50 & $0.1228\pm0.0041$ & $0.2368\pm0.0015$ \\
    100 & $0.1242\pm0.0046$ & $0.2374\pm0.0025$ \\
    200 & $0.1235\pm0.0071$ & $0.2373\pm0.0019$ \\
    500 & $0.1235\pm0.0039$ & $0.2380\pm0.0021$ \\
    1,000 & $0.1243\pm0.0037$ & $0.2387\pm0.0019$ \\
    Initial only & $0.1396\pm0.0077$ & $0.2457\pm0.0021$ \\
    \bottomrule
  \end{tabular}
\end{table}


\end{document}